\PassOptionsToPackage{dvipdfmx}{graphicx}
\documentclass[]{interact}
\usepackage{graphicx}
\usepackage{makecell}
\usepackage{multirow}
\usepackage{arydshln}

\usepackage{epstopdf}
\usepackage[caption=false]{subfig}

\usepackage[numbers,sort&compress]{natbib}
\bibpunct[, ]{[}{]}{,}{n}{,}{,}

\usepackage{color}
\usepackage[hyphens]{url}

\usepackage{amsmath,amssymb}
\usepackage{algorithm}
\usepackage{algpseudocode}

\usepackage{makecell}
\usepackage{arydshln}

\usepackage[numbers,sort&compress]{natbib}
\bibpunct[, ]{[}{]}{,}{n}{,}{,}

\theoremstyle{plain}
\newtheorem{theorem}{Theorem}[section]

\newtheorem{proposition}[theorem]{Proposition}

\theoremstyle{definition}

\theoremstyle{remark}

\begin{document}

\articletype{ARTICLE TEMPLATE}

\title{Covariance-Based Structural Equation Modeling in Small-Sample Settings with $p>n$}

\author{
\name{Hiroki Hasegawa\textsuperscript{a}\thanks{CONTACT Hiroki Hasegawa. Email: hasegawa.hiroki.tkb\_en@u.tsukuba.ac.jp},
Aoba Tamura\textsuperscript{a},
and Yukihiko Okada\textsuperscript{b,c,d}}
\affil{
\textsuperscript{a}Graduate School of Science and Technology, University of Tsukuba, Tsukuba, Ibaraki, Japan;
\textsuperscript{b}Institute of Systems and Information Engineering, University of Tsukuba, Tsukuba, Ibaraki, Japan;
\textsuperscript{c}Tsukuba Institute for Advanced Research, University of Tsukuba, Tsukuba, Ibaraki, Japan;
\textsuperscript{d}Center for Artificial Intelligence Research, University of Tsukuba, Tsukuba, Ibaraki, Japan
}
}

\maketitle

\begin{abstract}
Factor-based Structural Equation Modeling (SEM) relies on likelihood-based estimation assuming a nonsingular sample covariance matrix, which breaks down in small-sample settings with $p>n$. To address this, we propose a novel estimation principle that reformulates the covariance structure into self-covariance and cross-covariance components. The resulting framework defines a likelihood-based feasible set combined with a relative error constraint, enabling stable estimation in small-sample settings where $p>n$ for sign and direction. Experiments on synthetic and real-world data show improved stability, particularly in recovering the sign and direction of structural parameters. These results extend covariance-based SEM to small-sample settings and provide practically useful directional information for decision-making.
\end{abstract}


\section{Introduction}

Structural Equation Modeling (SEM) is a statistical method for representing the relationships between latent and observed variables through covariance structures and for statistically evaluating the extent to which a hypothesized causal structure fits the data \citep{MacCallum2000}. Although SEM provides a general framework for modeling relations between latent and observed variables, it can be broadly classified into two fundamentally different paradigms according to its theoretical foundation: Factor-based SEM, which formulates latent variables as random variables and aims to reproduce the covariance structure implied by the latent variable model, and Component-based SEM, which constructs latent variables as linear combinations of observed variables and places greater emphasis on predictive performance and explanatory power \citep{cho2022comparative}.

Factor-based SEM is a framework that explains the covariance structure among observed variables by means of a latent variable model, and its estimation is based on the agreement between the sample covariance matrix and the model-implied covariance matrix. A representative approach within this framework is covariance-based SEM (CB-SEM), which is typically implemented via maximum likelihood estimation that minimizes the discrepancy between these two covariance matrices \citep{Joreskog1967}. This estimation procedure relies on the classical asymptotic framework in which \( p \) is fixed and \( n \to \infty \), under conditions such as model identifiability and regularity of the covariance matrix, and is characterized by the fact that consistency and asymptotic normality are guaranteed \citep{wald1949note}. The term CB-SEM itself is commonly used in the literature to refer to this class of covariance structure–based estimation approaches (e.g., \cite{hair2025covariance}).

In contrast, Component-based SEM does not aim at the exact reproduction of the covariance structure, but rather ensures estimability even in small-sample or high-dimensional settings by constructing latent variables as linear combinations of observed variables. Representative examples include PLS-SEM and Generalized Structured Component Analysis (GSCA), both of which are based on estimation principles that do not depend on a likelihood function and therefore avoid the problems caused by covariance matrix singularity \citep{Hair2011,Hwang2004}. Component-based SEM adopts an approach in which constructs are defined as weighted sums of observed variables, and is therefore fundamentally different from Factor-based SEM, which estimates latent variables under a common factor model. As a result, confirmatory factor analysis, which tests whether the observed variables are generated by an assumed latent factor structure, is not conducted, and the validity of the measurement model is not explicitly examined. In other words, Component-based SEM defines constructs compositionally from observed variables, whereas Factor-based SEM assumes the existence of latent constructs and estimates them; the two therefore rest on fundamentally different modeling principles.

In addition to these approaches, there exists a Bayesian approach. Bayesian SEM has been reported to enable stable estimation even in small-sample settings by introducing prior distributions on latent variables and parameters \cite{lee2007structural}. We emphasize that Bayesian SEM differs from Factor-based SEM and Component-based SEM in that it represents an estimation framework rather than a classification based on model structure.

Based on the above discussion, SEM can be organized into three categories: Factor-based SEM and Component-based SEM under the frequentist paradigm, and Bayesian SEM under the Bayesian paradigm. Existing methods are summarized according to this classification in Table~\ref{table1}. As shown in Table~\ref{table1}, most Factor-based SEM approaches assume settings in which the dimensionality is smaller than or at most comparable to the sample size. Moreover, even SEM based on shrinkage targets, which has recently been proposed as an approach for small-sample settings, still assumes that the dimensionality is smaller than the sample size \cite{de2023model}. The comparison summarized in Table~\ref{table1} indicates that many situations referred to as ``small-sample'' settings in the existing literature correspond in practice to cases in which the number of observed variables \( p \) is smaller than the sample size \( n \), that is, \( p < n \). In contrast, the present study explicitly distinguishes between small-sample settings with \( p > n \) and those with \( p < n \) on the basis of the relationship between the number of observed variables and the sample size. The distinction between the regimes $p>n$ and $p<n$ is not merely descriptive, but rather constitutes an essential factor that determines the theoretical validity of estimation and inference.

\begin{table}[htbp]
\centering
\caption{Summary of methods. Here, Frequentist and Bayesian are abbreviated as Freq. and Bayes., respectively.}
\label{table1}
\begin{tabular}{lllll}
\hline
Method & \makecell{1st Author \\ (Year)} & \makecell{Paradigm \\ Freq./Bayes.} & \makecell{Estimation \\ Principle} & \makecell{Dimensional \\ Regime} \\
\hline
\multicolumn{5}{l}{\textit{Factor-based SEM}} \\
SEM & J\"oreskog (1970)\cite{joreskog1970general} & Freq. & Likelihood & $p < n$ \\
Bootstrap SEM & Bollen (1992)\cite{bollen1992bootstrapping} & Freq. & Resampling & $p < n$ \\
Robust SEM & Satorra (1994)\cite{satorra1994corrections} & Freq. & Robust loss & $p < n$ \\
Ridge-SEM & Yuan (2011)\cite{yuan2011ridge} & Freq. & L2 penalty & $p \approx n$ \\
LASSO-SEM & Jacobucci (2016)\cite{Jacobucci2016} & Freq. & L1 penalty & $p \approx n$ \\
Shrinkage-SEM & De Jonckere (2023)\cite{de2023model} & Freq. & Shrinkage & $p < n$ \\
\textbf{Prop. Method} & \textbf{This paper} & \textbf{Freq.} & \textbf{\makecell{HDLSS\\ Approach}} & $\mathbf{p>n}$\\
\hdashline
\multicolumn{5}{l}{\textit{Component-based SEM}} \\
PLS-SEM & Lohmoller(2013)\cite{lohmoller2013latent} & Freq. & Projection & $p > n$ \\
GSCA & Hwang (2004)\cite{Hwang2004} & Freq. & ALS & $p > n$ \\
\hdashline
\multicolumn{5}{l}{\textit{Bayesian SEM}} \\
Bayesian SEM & Lee (2007)\cite{lee2007structural} & Bayes. & Posterior & $p \ge n$ \\
\hline
\end{tabular}
\end{table}

However, in Factor-based SEM under small-sample settings with \( p > n \), the assumptions underlying estimation and inference based on asymptotic theory collapse structurally. The essential issue is not merely an increase in dimensionality, but the fact that, in the limit \( p/n \to \infty \), the assumptions required for conventional statistical inference cease to hold simultaneously. Specifically, because the sample covariance matrix \( S \) has rank at most \( n \), it is necessarily singular, and the Wishart assumption no longer holds. As a consequence, the likelihood function involving terms such as \( \log|S| \) becomes undefined or unstable, rendering likelihood-based estimation and inference procedures such as the chi-square test and the Wald test inapplicable \citep{johnstone2001distribution}. Furthermore, because the observed data are confined to a low-dimensional subspace, some parameters of the model-implied covariance matrix \( \Sigma(\theta) \) become unidentifiable, and the Fisher information matrix also becomes singular. In such singular models, the asymptotic normality of the maximum likelihood estimator does not hold, and Wilks' theorem is no longer applicable; consequently, the likelihood ratio statistic does not follow a \( \chi^2 \) distribution \citep{Chetelat2012,Rotnitzky2000,Watanabe2010,Drton2009}. Although extensions such as small-sample corrections \citep{Bentler1999}, regularized estimation \citep{yuan2011ridge,Jacobucci2016}, and robust estimation \citep{Yuan2007} have been proposed, none of them fundamentally resolve the intrinsic limitations of the covariance-structure-based framework itself.

Moreover, the nature of this breakdown differs qualitatively depending on the asymptotic regime of the dimensionality ratio \( p/n \). When \( p/n \to c \in (0,\infty) \), the eigenvalues of the sample covariance matrix exhibit systematic distortion, but correction is possible using random matrix theory \citep{johnstone2001distribution,paul2007asymptotics}. By contrast, in small-sample settings with \( p > n \) where \( p/n \to \infty \), the sample covariance matrix becomes intrinsically low-rank and unstable, and consistent estimation of the eigenspace structure of the population covariance matrix becomes impossible. Under the regime $p/n \to \infty$, estimators of both eigenvalues and eigenvectors lose consistency\citep{jung2009pca}. Therefore, the challenge of SEM in small-sample settings with \( p > n \) should be understood not merely as a deterioration in estimation accuracy, but as a structural problem in which the very assumptions required for statistical inference no longer hold.

In this study, under a problem setting defined by the relationship between the sample size \( n \) and the number of observed variables \( p \) (Figure~\ref{fig:d_and_n_classification}), we propose a novel covariance-based SEM framework that remains applicable even in small-sample settings with \( p > n \), where likelihood-based estimation and statistical inference fail. The proposed method re-formulates the estimation problem in a way that does not rely on conventional maximum likelihood estimation, thereby making it possible to estimate latent variables and to evaluate theoretical models.

\begin{figure}[h]
\centering
\includegraphics[width=0.75\linewidth]{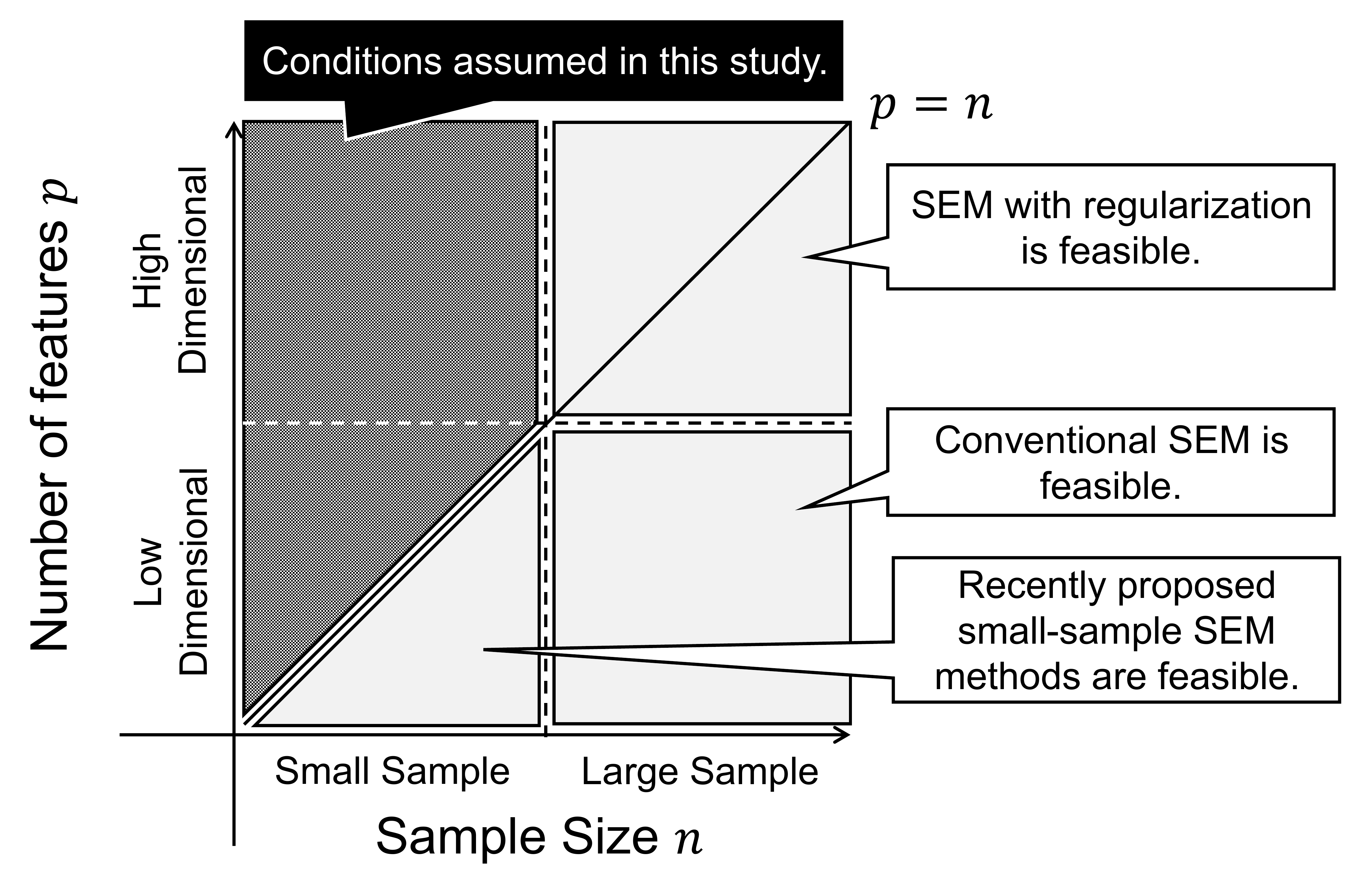}
\caption{Conceptual illustration of the applicability regions of SEM methods based on the relationship between the sample size $n$ and the number of features $p$. Conventional SEM and recently proposed small-sample SEM methods are primarily developed under the regime $p < n$, whereas regularized SEM is applicable not only in $p < n$ but also in settings where $p \approx n$. In contrast, this study focuses on the small-sample settings with \( p > n \), which lies outside the scope of these existing approaches.}
\label{fig:d_and_n_classification}
\end{figure}

\section{Proposed Approach}

We adopt a novel estimation framework that decomposes the covariance structure based on its estimability. Unlike the conventional two-step SEM \cite{anderson1988structural}, which sequentially estimates the measurement model and the structural model, our approach separates covariance information into self-covariance and cross-covariance components and assigns distinct statistical roles to each.

Specifically, the observed variables are partitioned into two blocks with dimensions $p_1$ and $p_2$, and we assume that $n < p_1 + p_2$ while $n > p_1$ and $n > p_2$. Under this setting, the self-covariance within each block can be stably estimated, whereas the full covariance matrix becomes singular. Our method exploits this partial estimability by constructing a feasible parameter set based on classical likelihood-based goodness-of-fit for the self-covariance blocks, while introducing constraints on the cross-covariance block based on scalar quantities that are stably estimable under high-dimensional settings, rather than enforcing direct matrix-level agreement.

For the estimation of the self-covariance structure, we assume a one-factor model for each block. That is, the latent variables are scalar, and the corresponding factor loading vectors are $\Lambda_x \in \mathbb{R}^{p_1}$ and $\Lambda_y \in \mathbb{R}^{p_2}$. The one-factor restriction is introduced to ensure identifiability. In general, multi-factor models suffer from rotational indeterminacy, i.e., $\Lambda \mapsto \Lambda Q$ with $Q^\top Q = I$, under which the covariance structure remains invariant. As a result, the parameterization is non-unique, leading to instability in the feasible set derived from self-covariance.

Furthermore, for the cross-covariance, we assume a rank-1 structure
\[
\Sigma_{xy} = \beta \Lambda_x \Lambda_y^\top,
\]
thereby reducing estimation to a scalar parameter $\beta$. In contrast, multi-factor models yield $\Sigma_{xy} = \Lambda_x B \Lambda_y^\top$, where the increased degrees of freedom deteriorate identifiability and estimation stability.

In summary, the proposed method decomposes the unstable estimation problem of the full covariance matrix into (i) construction of a candidate set based on self-covariance and (ii) selection based on cross-covariance constraints. By combining this structure with minimal identifiability constraints, stable estimation is achieved in small-sample settings.

\subsection{Model Formulation}

We describe the covariance structure of observed variables based on the LISREL (Linear Structural Relations) model. The observed variables are partitioned into two blocks, $\mathbf{x} \in \mathbb{R}^{p_1}$ and $\mathbf{y} \in \mathbb{R}^{p_2}$, and the measurement model is given by
\[
\mathbf{x} = \Lambda_x \xi + \delta,\quad
\mathbf{y} = \Lambda_y \eta + \epsilon.
\]
Here, $\xi \in \mathbb{R}$ and $\eta \in \mathbb{R}$ are latent variables, and $\delta \in \mathbb{R}^{p_1}$ and $\epsilon \in \mathbb{R}^{p_2}$ are measurement errors. The factor loading matrices are $\Lambda_x \in \mathbb{R}^{p_1 \times 1}$ and $\Lambda_y \in \mathbb{R}^{p_2 \times 1}$.

The structural model is defined as
\[
\eta = B \eta + \Gamma \xi + \zeta,
\]
which can be rewritten as
\[
\eta = (I-B)^{-1}\Gamma \xi + (I-B)^{-1}\zeta
\]
when $(I-B)$ is invertible.

We assume the following covariance structure with zero means:
\[
\mathrm{Cov}(\xi) = \Phi,\quad
\mathrm{Cov}(\zeta) = \Psi,\quad
\mathrm{Cov}(\delta) = \Theta_\delta,\quad
\mathrm{Cov}(\epsilon) = \Theta_\epsilon,
\]
and assume that $\xi, \zeta, \delta,$ and $\epsilon$ are mutually independent.

Under these assumptions, the theoretical covariance matrix of $\mathbf{z} = (\mathbf{x}^\top, \mathbf{y}^\top)^\top$ is given by
\[
\Sigma(\theta) =
\begin{pmatrix}
\Sigma_{xx}(\theta) & \Sigma_{xy}(\theta) \\
\Sigma_{yx}(\theta) & \Sigma_{yy}(\theta)
\end{pmatrix},
\]
where each block is expressed as
\begin{align*}
\Sigma_{xx}(\theta) &= \Lambda_x \Phi \Lambda_x^\top + \Theta_\delta,\\
\Sigma_{yy}(\theta)
&= \Lambda_y \left(
(I-B)^{-1}\Gamma \Phi \Gamma^\top (I-B)^{-\top}
+ (I-B)^{-1}\Psi (I-B)^{-\top}
\right)\Lambda_y^\top + \Theta_\epsilon,\\
\Sigma_{xy}(\theta)
&= \Lambda_x \Phi \Gamma^\top (I-B)^{-\top} \Lambda_y^\top.
\end{align*}

In general, LISREL models inherently suffer from identifiability issues. Specifically, it is well known that there exist several types of invariances, including rotational invariance arising from linear transformations of latent variables, scale invariance due to rescaling, non-uniqueness in the decomposition of cross-covariance (cross-decomposition invariance), and sign invariance due to sign flipping. These invariances imply that infinitely many parameter configurations yield the same covariance structure, making unique estimation impossible under the general model.

To mitigate these issues, we impose a one-factor model for each block. The resulting one-factor LISREL model is formulated as follows. The measurement model is
\[
\mathbf{x} = \Lambda_x \xi + \delta,\quad
\mathbf{y} = \Lambda_y \eta + \epsilon,
\]
and the structural model is given by
\[
\eta = \beta \xi + \zeta.
\]

We further assume the following covariance structure:
\[
\mathrm{Var}(\xi) = 1,\quad
\mathrm{Var}(\zeta) = \Psi,\quad
\mathrm{Cov}(\delta) = \Theta_\delta,\quad
\mathrm{Cov}(\epsilon) = \Theta_\epsilon,
\]
where $\Theta_\delta$ and $\Theta_\epsilon$ are diagonal matrices.

For identifiability, we impose the constraints
\[
\lambda_{x1} = 1,\quad \lambda_{y1} = 1.
\]

Under this formulation, the theoretical covariance matrix is given by
\begin{align*}
\Sigma_{xx} &= \Lambda_x \Lambda_x^\top + \Theta_\delta,\\
\Sigma_{yy} &= \Lambda_y (\beta^2 + \Psi)\Lambda_y^\top + \Theta_\epsilon \hspace{0.5em}(=\tau \Lambda_y \Lambda_y^\top + \Theta_\epsilon),\\
\Sigma_{xy} &= \beta \Lambda_x \Lambda_y^\top.
\end{align*}

With the introduction of the one-factor structure, rotational invariance is eliminated, and the remaining identifiability issues are limited to scale invariance and cross-decomposition invariance. The scale invariance is resolved by fixing $\mathrm{Var}(\xi)=1$, which removes the ambiguity in $\Phi \beta$, and further by imposing $\lambda_{x1}=1$ and $\lambda_{y1}=1$, which fixes the scale of the factor loadings. These constraints also implicitly fix the sign, thereby eliminating sign invariance.

Moreover, by assuming that the error covariance matrices $\Theta_\delta$ and $\Theta_\epsilon$ are diagonal, the factor loading vectors $\Lambda_x$ and $\Lambda_y$ are (up to fixed scale and sign) almost uniquely determined from the self-covariance structures $\Sigma_{xx}$ and $\Sigma_{yy}$. Therefore, under the proposed setting, it becomes possible to stably estimate $\Lambda_x$ and $\Lambda_y$ based solely on the self-covariance components.

\subsection{Estimation of Self-Covariance}

From the observed data $\mathbf{z}$, we construct the sample covariance matrix $S$ as follows:
\[
S =
\begin{pmatrix}
S_{xx} & S_{xy} \\
S_{yx} & S_{yy}
\end{pmatrix}.
\]
Among these, the sample self-covariance matrices $S_{xx}$ and $S_{yy}$ represent the variance structure within each variable block.

In high-dimensional settings, the estimation error of the sample covariance matrix increases with dimensionality, and stable estimation requires an appropriate relationship between the dimension and the sample size \citep{bickel2008regularized}. We assume $n > p_1$ and $n > p_2$. Under this assumption, if the observed data follow a continuous distribution, $S_{xx}$ and $S_{yy}$ are positive definite with probability one, allowing the application of classical maximum likelihood estimation based on the Wishart distribution.

The negative log-likelihood functions for each self-covariance block are given by
\begin{align*}
    \ell_x(\Lambda_x, \Theta_\delta)
    &= \log \lvert \Sigma_{xx}(\Lambda_x, \Theta_\delta) \rvert
    + \operatorname{tr}\bigl(S_{xx}\Sigma^{-1}_{xx}(\Lambda_x, \Theta_\delta)\bigr),\\
    \ell_y(\Lambda_y, \Theta_\epsilon, \tau)
    &= \log \lvert \Sigma_{yy}(\Lambda_y, \Theta_\epsilon, \tau) \rvert
    + \operatorname{tr}\bigl(S_{yy}\Sigma^{-1}_{yy}(\Lambda_y, \Theta_\epsilon, \tau)\bigr).
\end{align*}
Here, the theoretical self-covariance matrices are defined as
\[
\Sigma_{xx}(\Lambda_x, \Theta_\delta) = \Lambda_x \Lambda_x^\top + \Theta_\delta,\quad
\Sigma_{yy}(\Lambda_y, \Theta_\epsilon, \tau) = \tau \Lambda_y \Lambda_y^\top + \Theta_\epsilon.
\]

Accordingly, the maximum likelihood estimator for the self-covariance component is defined as
\begin{align}
\label{ml-theta-self}
\widehat{\theta}
&=
\operatorname*{arg\,min}_{\Theta}
\left\{
\ell_x(\Lambda_x, \Theta_\delta)
+
\ell_y(\Lambda_y, \Theta_\epsilon, \tau)
\right\},\\
\widehat{\theta}
&:=
\left(
\widehat{\Lambda}_x, \widehat{\Theta}_{\delta},
\widehat{\Lambda}_y, \widehat{\Theta}_\epsilon,
\widehat{\tau}
\right),\notag\\
\Theta_{\mathrm{self}}^{\mathrm{param}}
&=
\left\{
(\Lambda_x, \Theta_\delta,\Lambda_y, \Theta_\epsilon, \tau)
:\ 
\lambda_{x1}=1,\ 
\lambda_{y1}=1,\ 
\Theta_\delta, \Theta_\epsilon \text{ are diagonal positive definite},\
\tau>0
\right\}.\notag
\end{align}

In practice, due to nonconvex optimization, the maximum likelihood solution for the self-covariance component is not necessarily unique, as it may depend on initialization and numerical procedures. To address this issue, we define the solution not as a single point estimate, but as a set of parameters in the vicinity of the optimal solution.

Specifically, we define the feasible set for the self-covariance component as
\[
\Theta_{\mathrm{self}}
=
\left\{
\theta \in \Theta_{\mathrm{self}}^{\mathrm{param}}
:\ 
\ell_{\mathrm{self}}(\theta)
\le
\ell_{\mathrm{self}}(\widehat{\theta}_{\mathrm{self}})
+
\epsilon_n
\right\},
\]
where $\ell_{\mathrm{self}}(\theta)=\ell_x+\ell_y$ denotes the negative log-likelihood for the self-covariance component, $\widehat{\theta}_{\mathrm{self}}$ is the optimal solution obtained via numerical optimization, and $\epsilon_n$ is a tolerance parameter that accounts for estimation error.

This formulation allows us to construct a set of statistically equivalent solutions that are robust to optimization instability and local minima. The threshold $\epsilon_n$ is determined via a bootstrap procedure. Details of this procedure are provided in Section~\ref{subsection:hyperparameter}.

\subsection{Estimation of Cross-Covariance}

Previous studies (e.g., \cite{Umino2025}) have shown that the sample cross-covariance matrix $S_{xy}$ is not a stable estimator in high-dimensional, low-sample-size settings. In particular, under the regime $n < p_1 + p_2$, $S_{xy}$ is generally inconsistent and highly sensitive to noise, making it unsuitable for direct use in estimation.

To address this issue, \cite{Umino2025} proposed a thresholded estimator of the cross-covariance matrix, denoted by $\widetilde{\Sigma}_{xy}$. This estimator removes noise components by imposing a sparsity assumption on the structure of the cross-covariance matrix, and it has been shown that, under appropriate conditions, it achieves consistency in the sense of relative error:
\begin{equation}
\label{eq:Umino2025equation}
\frac{\|\widetilde{\Sigma}_{xy} - \Sigma_{xy}\|_F^2}{\Delta_{xy}} = o_P(1),
\qquad
\Delta_{xy} := \|\Sigma_{xy}\|_F^2.
\end{equation}

However, this result in Equation~\eqref{eq:Umino2025equation} does not hold universally. For example, in \cite{Umino2025}, consistency is established under sparsity conditions such as
\begin{equation}
\label{eq:sparsity_condition}
\begin{aligned}
    &\text{there exists a nonempty subset } C_1 \subset C \text{ such that} \\
    &\sum_{(r,s)\in C_1}\frac{\sigma^2_{(r,s)}}{\Delta_{xy}} = 1+o(1),\quad
    \liminf_{p\to\infty} \left(\min_{(r,s)\in C_1}\sigma^2_{(r,s)}\right)>0,\\
    &\max_{(r,s)\in C\setminus C_1}\sigma^2_{(r,s)}=o(1),
\end{aligned}
\end{equation}
together with conditions such as $p_1p_2/n^4 \to 0$ or $\log(p)/n \to 0$ under sub-exponential distributions.

We adopt $\widetilde{\Sigma}_{xy}$ as an estimator of the cross-covariance matrix under such structural assumptions. On the other hand, under the one-factor LISREL model, the theoretical cross-covariance matrix is given by
\[
\Sigma_{xy}(\theta,\beta)
=
\beta \Lambda_x(\theta)\Lambda_y(\theta)^\top.
\]

Based on this structure, we construct a feasible set for $(\theta,\beta)$ by imposing constraints derived from the cross-covariance structure on the feasible set $\Theta_{\mathrm{self}}$ obtained from the self-covariance component. Specifically, we define
\[
\mathcal{F}_{\mathrm{cross}}
=
\left\{
(\theta,\beta)\in \Theta_{\mathrm{self}}\times\mathbb{R}
:\ 
\frac{
\|\widetilde{\Sigma}_{xy}
-
\beta \Lambda_x(\theta)\Lambda_y(\theta)^\top
\|_F^2
}{\Delta_{xy}}
\le
\xi_n
\right\},
\]
where $\xi_n$ is a tolerance parameter that controls the allowable estimation error.

By imposing constraints based on relative error rather than absolute error, the proposed formulation effectively controls the impact of estimation error in high-dimensional settings. Furthermore, by combining the sparsity assumption induced by thresholding with the rank-1 structure implied by the one-factor model, the proposed approach achieves both stable estimation of the cross-covariance matrix and interpretable structural representation.

The threshold $\xi_n$ is determined via a bootstrap procedure based on the statistical quantities proposed in \cite{Umino2025}. Details of this procedure are provided in Section~\ref{subsection:hyperparameter}.

\subsection{Parameter Estimation}

In the previous sections, we constructed the feasible set $\mathcal{F}_{\mathrm{cross}}$ based on both the self-covariance and cross-covariance structures. In this section, we describe how the final estimator is selected from this feasible set.

As a measure of goodness-of-fit between the observed covariance matrix and the model-implied covariance matrix, we adopt the Standardized Root Mean Square Residual (SRMR). The SRMR quantifies the average magnitude of standardized residuals between the sample covariance matrix $S$ and the model covariance matrix $\Sigma(\theta,\beta)$, where smaller values indicate better model fit. Specifically, SRMR is defined as
\[
\mathrm{SRMR}(\theta,\beta)
=
\left(
\frac{2}{(p_1+p_2)(p_1+p_2+1)}
\sum_{i \le j}
\left(
\frac{
S_{ij}
-
\Sigma_{ij}(\theta,\beta)
}{
\sqrt{S_{ii}S_{jj}}
}
\right)^2
\right)^{1/2}.
\]

We define the final estimator as the parameter that minimizes SRMR over the feasible set $\mathcal{F}_{\mathrm{cross}}$, i.e.,
\[
(\widehat{\theta}, \widehat{\beta})
=
\operatorname*{arg\,min}_{(\theta,\beta)\in \mathcal{F}_{\mathrm{cross}}}
\mathrm{SRMR}(\theta,\beta).
\]

This formulation selects the parameter that achieves the best global fit to the covariance structure among those satisfying the constraints derived from both self-covariance and cross-covariance components. By restricting the search to the feasible set, the proposed method mitigates the effects of local estimation errors and optimization instability, while balancing statistical consistency and practical goodness-of-fit.

\subsection{Statistical Inference}
\label{subsection:test}

In the proposed method, the estimator involves nonlinear and discontinuous operations, such as sparsification and feasibility-based selection. As a result, conventional Wald-type inference based on asymptotic normality is not applicable. Moreover, due to the combination of feasible set construction and nonconvex optimization, deriving the asymptotic distribution of the estimator analytically is highly challenging.

To address this issue, we employ a bootstrap procedure to evaluate the uncertainty of the coefficient $\beta$.

Specifically, we generate bootstrap samples $\{\mathbf{z}_i^{*(b)}\}_{i=1}^n$ by resampling with replacement from the observed data $\{\mathbf{z}_i\}_{i=1}^n$. For each bootstrap sample, we apply the proposed method to obtain the estimator $\hat{\beta}^{*(b)}$. That is, for each $b = 1,\dots,B$, we compute
\[
\hat{\beta}^{*(b)}
=
\operatorname*{arg\,min}_{(\theta,\beta)\in \mathcal{F}_{\mathrm{cross}}^{*(b)}}
\mathrm{SRMR}(\theta,\beta).
\]

Based on the empirical distribution of $\{\hat{\beta}^{*(b)}\}_{b=1}^B$, confidence intervals can be constructed, for example, using the percentile method. Hypothesis testing for $H_0:\beta=0$ is then conducted by examining whether zero is contained within the resulting confidence interval.

\subsection{Hyperparameter Tuning}
\label{subsection:hyperparameter}

\subsubsection{Estimation of $\epsilon_n$}
\label{subsubsection:epsilon}

The parameter $\epsilon_n$ is introduced to quantify the estimation error in the self-covariance component. For each bootstrap sample $(X^{(b)}, Y^{(b)})$, we compute the corresponding maximum likelihood estimator $\theta^{(b)}$, and define
\[
\Delta_b
=
\ell_{\text{self}}(\widehat{\theta}; S^{(b)})
-
\ell_{\text{self}}(\widehat{\theta}^{(b)}; S^{(b)}),
\]
where $S^{(b)}$ denotes the covariance matrix computed from the $b$-th bootstrap sample.

We then define $\epsilon_n$ as the $(1-\alpha)$-quantile of $\{\Delta_b\}_{b=1}^{B}$:
\[
\epsilon_n
=
\operatorname{Quantile}_{1-\alpha}
\left(
\{\Delta_b\}_{b=1}^{B}
\right).
\]

\subsubsection{Estimation of $\xi_n$}
\label{subsubsection:xi}

Next, $\xi_n$ represents a probabilistic tolerance level that controls the allowable deviation in the cross-covariance component. We define this tolerance in terms of the upper tail probability of the true error distribution. Let
\[
T_n
=
\frac{
\|\widetilde{\Sigma}_{xy}
-
\Sigma_{xy}
\|_F^2
}{
\Delta_{xy}
},
\]
then $\xi_n$ is defined as
\[
\xi_n
=
\inf
\left\{
t :
\mathrm{Pr}(T_n \le t) \ge 1-\alpha
\right\}.
\]

Under the LISREL model with Gaussian assumptions, the data belong to the sub-exponential family. Therefore, it suffices to consider the regime $\log(p)/n \to 0$ as $p,n\to\infty$. In particular, we focus on the following result from \cite{Umino2025}:
\begin{equation}
\label{Umino:eq4}
\frac{\|\widetilde{\Sigma}_{xy}-\Sigma_{xy}\|_F^2}{\Delta_{xy}}
=
O_P\left(
\eta
+
\frac{\log (p)}{n}
+
\frac{\operatorname{tr}(\Sigma^2_{xx})^{1/4}\operatorname{tr}(\Sigma^2_{yy})^{1/4}}{\sqrt{n \Delta_{xy}}}
\right).
\end{equation}

Here, $\eta$ is defined as
\begin{equation}
\label{Umino:eq5}
\eta
=
\sum_{(r,s)\in C_2}
\frac{\sigma^2_{(r,s)}}{\Delta_{xy}},
\end{equation}
where $C_2 = C \setminus C_1$.

Equation~\eqref{Umino:eq4} shows that the relative error consists of multiple components. Specifically, the first term corresponds to $\eta$, representing the proportion of energy removed by thresholding. The second term, proportional to $\log(p)/n$, reflects the estimation error in high-dimensional settings. The third term depends on $\operatorname{tr}(\Sigma_{xx}^2)$ and $\operatorname{tr}(\Sigma_{yy}^2)$, capturing the interaction between covariance magnitude and sample size.

Under the sparsity condition~\eqref{eq:sparsity_condition}, the energy of the cross-covariance matrix is concentrated on a finite number of dominant components, implying $\eta \to 0$. Thus, $\eta$ becomes asymptotically negligible and can be absorbed into higher-order error terms.

Based on this result, we define $\xi_n$ according to the rate in Equation~\eqref{Umino:eq4}. The quantities $\operatorname{tr}(\Sigma_{xx}^2)$ and $\operatorname{tr}(\Sigma_{yy}^2)$ can be consistently estimated using the $W_n$-type estimator proposed in \cite{yata2013correlation}:
\begin{equation}
\label{yata2013estimator}
W_n
=
\frac{2u_n}{n(n-1)}
\sum_{i<j}
\left[
\left(x_i - \bar{x}_{ij}^{(1)}\right)^\top
\left(x_j - \bar{x}_{ij}^{(2)}\right)
\right]^2,
\end{equation}
where $u_n=\frac{n_1n_2}{(n_1-1)(n_2-1)}$ is a finite-sample correction factor, and $\bar{x}_{ij}^{(1)}, \bar{x}_{ij}^{(2)}$ denote sample means computed from two disjoint subsets constructed so that $x_i$ and $x_j$ belong to different subsets.

Furthermore, by employing an estimator $\widehat{\Delta}_{xy}$ of $\Delta_{xy}$ based on the ECDM method, $\xi_n$ can be explicitly written as
\begin{equation}
\label{eq:eta_n}
\xi_n
=
C\left(
\frac{\log(p)}{n}
+
\frac{W_1^{1/4}W_2^{1/4}}{\sqrt{n\widehat{\Delta}_{xy}}}
\right),
\end{equation}
where $C>0$ is a constant.

The above formulation holds under asymptotic conditions. In finite samples, the constant $C$ in the theoretical upper bound is unknown. Therefore, instead of specifying a fixed threshold manually, we adopt a data-driven approach that simultaneously determines the sparse structure and the error bound.

First, for the observed data $\{(\mathbf{x}_i,\mathbf{y}_i)\}_{i=1}^n$, we construct an estimator $\widehat{\Delta}_{xy}$ of the signal strength $\Delta_{xy} = \|\Sigma_{xy}\|_F^2$ using the ECDM method proposed in \cite{Umino2025}.

We then sort the entries of the sample cross-covariance matrix $S_{xy}$ in descending order of magnitude, and retain components until the cumulative squared sum exceeds $\widehat{\Delta}_{xy}$, thereby constructing a sparse estimator $\widetilde{\Sigma}_{xy}$. This corresponds to sparsification based on energy preservation rather than explicit thresholding.

Next, to evaluate the probabilistic scale of the estimation error, we employ a bootstrap procedure. For each bootstrap sample $b=1,\dots,B$, we compute $\widetilde{\Sigma}_{xy}^{(b)}$ and $\widehat{\Delta}_{xy}^{(b)}$ in the same manner, and define
\begin{equation}
\label{eq:T_bootstrap}
T^{(b)}
=
\frac{
\left\|
\widetilde{\Sigma}_{xy}^{(b)}
-
\widetilde{\Sigma}_{xy}
\right\|_F^2
}{
\widehat{\Delta}_{xy}^{(b)}
}.
\end{equation}

We further define the theoretical rate
\begin{equation}
\label{eq:r_np}
r_{n,p}
=
\frac{\log p}{n}
+
\frac{
W_1^{1/4}W_2^{1/4}
}{
\sqrt{n\widehat{\Delta}_{xy}}
}.
\end{equation}

Let $q_{1-\alpha}$ denote the $(1-\alpha)$ upper quantile of $\{T^{(b)}\}_{b=1}^B$:
\begin{equation}
\label{eq:q_upper}
q_{1-\alpha}
=
\operatorname{Quantile}_{1-\alpha}
\left(
T^{(1)},\dots,T^{(B)}
\right).
\end{equation}

Based on consistency with the theoretical rate, we define an estimator of the constant as
\begin{equation}
\label{eq:C_hat}
\widehat{C}
=
\frac{q_{1-\alpha}}{r_{n,p}}.
\end{equation}

To avoid overestimation due to small values of $\widehat{\Delta}_{xy}^{(b)}$ in finite samples, we impose an upper bound and define
\begin{equation}
\label{eq:C_trunc}
\widehat{C}_{\mathrm{trunc}}
=
\min\left(\widehat{C},\, C_{\max}\right).
\end{equation}

Finally, the estimated tolerance parameter is given by
\begin{equation}
\label{eq:estimate_eta_n}
\widehat{\xi}_n
=
\widehat{C}_{\mathrm{trunc}}
\left(
\frac{\log(p)}{n}
+
\frac{W_1^{1/4}W_2^{1/4}}{\sqrt{n\widehat{\Delta}_{xy}}}
\right).
\end{equation}

\subsection{Overview of the Proposed Method}

Figure~\ref{fig:overview} illustrates the overall framework of the proposed method. The method adopts a staged estimation procedure based on the distinct roles of self-covariance and cross-covariance.

\begin{figure}[h]
\centering
\includegraphics[width=0.9\linewidth]{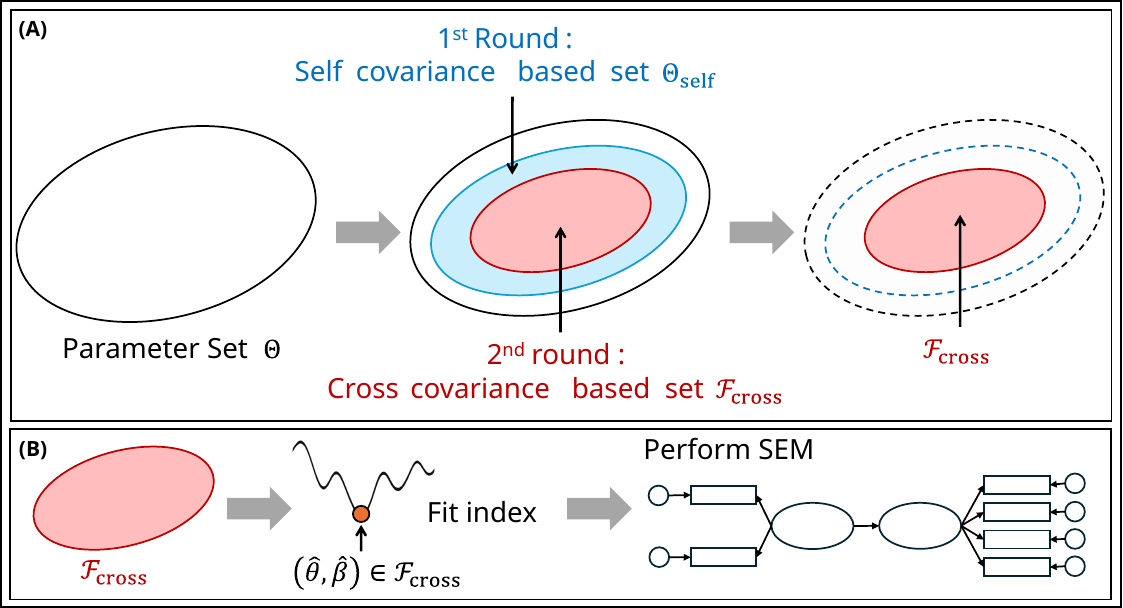}
\caption{Overview of the proposed method. The estimation procedure consists of three stages: (i) candidate generation based on self-covariance, (ii) filtering using cross-covariance, and (iii) final selection via SRMR minimization.}
\label{fig:overview}
\end{figure}

First, a factor model is applied to the self-covariance components to generate a set of candidate parameters that can be stably estimated. Next, this candidate set is refined by imposing consistency with the cross-covariance structure. Finally, the optimal parameter is selected based on a goodness-of-fit criterion, namely the Standardized Root Mean Square Residual (SRMR).

By explicitly separating the roles of the self-covariance and cross-covariance components, the proposed method effectively mitigates the instability inherent in estimation under small-sample settings with $p > n$.

The above estimation procedure is summarized as an algorithm from a computational perspective in Algorithm~\ref{alg:proposed_estimation}. The algorithm explicitly describes the sequence of operations, including candidate generation based on self-covariance, constraint imposition using cross-covariance, and final selection via a goodness-of-fit criterion, thereby providing a reproducible implementation of the proposed method.

\begin{algorithm}[h]
\caption{Proposed estimation procedure}
\label{alg:proposed_estimation}
\begin{algorithmic}[1]
\Require Observations $\{(\mathbf{x}_i,\mathbf{y}_i)\}_{i=1}^n$, bootstrap size $B$
\Ensure $(\widehat{\theta}, \widehat{\beta})$

\State Compute
\[
S =
\begin{pmatrix}
S_{xx} & S_{xy} \\
S_{yx} & S_{yy}
\end{pmatrix}
\]

\State $\widehat{\theta}_{\mathrm{self}}
\gets
\arg\min_{\theta \in \Theta_{\mathrm{param}}^{\mathrm{self}}}
\ell_{\mathrm{self}}(\theta)$

\State $\epsilon_n \gets \text{Bootstrap}(\ell_{\mathrm{self}})$

\State
\[
\Theta_{\mathrm{self}}
\gets
\left\{
\theta :
\ell_{\mathrm{self}}(\theta)
\le
\ell_{\mathrm{self}}(\widehat{\theta}_{\mathrm{self}})
+
\epsilon_n
\right\}
\]

\State $(\widetilde{\Sigma}_{xy}, \widehat{\Delta}_{xy}) \gets (\text{Sparse}(S_{xy}), \text{ECDM}(S_{xy}))$

\State $\xi_n \gets \text{Bootstrap}(\widetilde{\Sigma}_{xy}, \widehat{\Delta}_{xy})$

\State
\[
\mathcal{F}_{\mathrm{cross}}
\gets
\left\{
(\theta,\beta) :
\theta \in \Theta_{\mathrm{self}},
\;
\frac{
\left\|
\widetilde{\Sigma}_{xy}
-
\beta \Lambda_x(\theta)\Lambda_y(\theta)^\top
\right\|_F^2
}{
\widehat{\Delta}_{xy}
}
\le
\xi_n
\right\}
\]

\State
\[
(\widehat{\theta},\widehat{\beta})
\gets
\arg\min_{(\theta,\beta)\in \mathcal{F}_{\mathrm{cross}}}
\mathrm{SRMR}(\theta,\beta)
\]

\State \Return $(\widehat{\theta}, \widehat{\beta})$
\end{algorithmic}
\end{algorithm}

\section{Simulation Study}

We conduct simulations under two scenarios: (i) data generated under conditions where the assumptions in \cite{Umino2025} hold, and (ii) data generated under violations of these structural assumptions to evaluate robustness. We refer to the former as Case 1 and the latter as Case 2.

\subsection{Case 1}
\label{subsection:pattern1}

We generate synthetic data according to a single-factor LISREL model. Specifically, let $\{x_1,\cdots ,x_n\}=X\in\mathbb{R}^{n\times p_1}$ and $\{y_1,\cdots, y_n\}=Y\in\mathbb{R}^{n\times p_2}$, where $x\in X$ and $y\in Y$. The data-generating process is defined as
\begin{equation}
    x = \Lambda_x \xi + \delta,\quad
    y = \Lambda_y \eta + \epsilon,\quad
    \eta = \beta \xi + \zeta.
\end{equation}

Here, $\xi,\eta\in\mathbb{R}$ are latent variables, $\Lambda_x\in \mathbb{R}^{p_1}$ and $\Lambda_y\in \mathbb{R}^{p_2}$ are factor loading vectors, and $\beta$ is the structural coefficient. The noise terms $\delta, \epsilon,$ and $\zeta$ are mutually independent Gaussian random variables with mean zero.

To incorporate sparsity in the factor loadings, we assign strong signals to a finite number of leading components, while the remaining components are set to small values. Specifically, we define
\begin{equation}
\label{Pattern1:model}
\lambda_{x,r} =
\begin{cases}
1 & r=1,2, \\
a\cdot p_1^{-\alpha} & r \ge 3,
\end{cases}
\quad
\lambda_{y,s} =
\begin{cases}
1 & s=1,2, \\
b\cdot p_2^{-\alpha} & s \ge 3,
\end{cases}
\quad
a,b>0,\ \frac{1}{2}<\alpha<1.
\end{equation}

This construction satisfies the sparsity condition~\eqref{eq:sparsity_condition}. Details are provided in Appendix~\ref{appendixA}.

In the simulations, we adopt the following parameter configuration:
\begin{align}
\lambda_{x,r} &=
\begin{cases}
1 & r = 1,2, \\
0.3\,p_1^{-3/4} & r \ge 3,
\end{cases}
\qquad
\lambda_{y,s} =
\begin{cases}
1 & s = 1,2, \\
0.3\,p_2^{-3/4} & s \ge 3,
\end{cases}\\
\Theta_{\delta} &= 0.5\, I_{p_1}, \qquad
\Theta_{\varepsilon} = 0.5\, I_{p_2}, \qquad
\psi = 0.4.
\end{align}

Under this setup, the data are governed by a small number of dominant common factors while maintaining homogeneous noise across all variables. This allows us to construct a data-generating model that simultaneously captures sparsity and positive definiteness.

To investigate the effect of high-dimensional settings, we fix the sample size at $n=10$ and vary the dimensionality as $p_1=p_2=2,3,\cdots,9$. We then classify the experimental results based on the following two structural regimes.

We consider two settings defined by the relationship between the sample size $n$ and the dimensions $p_1, p_2$. Setting 1 corresponds to $n > p_1 + p_2$. In this regime, the sample covariance matrix is full-rank, and all methods, including conventional SEM, can be applied stably.

Setting 2 corresponds to $n > p_1, n > p_2$, but $n \le p_1 + p_2$. In this case, the self-covariance matrices of each block remain nonsingular, whereas the full covariance matrix becomes singular. As a result, conventional SEM methods that rely on the joint covariance structure become unstable, while the proposed method, which exploits the self-covariance structure, is expected to demonstrate its advantage.

By analyzing these two settings, our objective is not merely to observe performance degradation as dimensionality increases, but rather to identify the structural conditions under which each method fails. Thus, the focus of this study is not solely on estimation accuracy, but on evaluating the effectiveness of each method from the perspective of structural identifiability.

The classification of these settings is summarized in Table~\ref{tab:case_definition}.

\begin{table}[h]
\centering
\caption{Structural regimes based on dimensionality conditions}
\label{tab:case_definition}
\renewcommand{\arraystretch}{1.2}
\begin{tabular}{c c c c}
\hline
Setting & Condition & Structural property & $p_1,p_2$ \\ \hline

Setting 1
& $n > p_1 + p_2$
& Full covariance matrix is nonsingular
& $2,3,4$ \\

Setting 2
& \begin{tabular}{c}
$n > p_1,\; n > p_2,$ \\
$n \le p_1 + p_2$
\end{tabular}
& \begin{tabular}{c}
Self-covariance is nonsingular, \\
but full covariance is singular
\end{tabular}
& $5,6,7,8,9$ \\

\hline
\end{tabular}
\end{table}

\subsubsection{Comparison Metrics}

We compare the proposed method with conventional SEM, L1-regularized SEM (L1-SEM), and L2-regularized SEM (L2-SEM). For L1-SEM and L2-SEM, the regularization parameter is selected from the set $\{0.001, 0.01, 0.1, 1, 10, 100, 1000\}$.

In small-sample settings, cross-validation tends to be unstable. Therefore, instead of cross-validation, we evaluate each candidate parameter using five independent trials based on different random seeds. The root mean squared error (RMSE) is computed for each candidate, and the parameter achieving the smallest RMSE is selected as the final regularization parameter.

In regularized SEM, the choice of the regularization parameter has a significant impact on estimation accuracy. Thus, in this study, the parameter is determined via pre-tuning over a candidate set. However, such tuning procedures are inherently data-dependent and may not be stable in practical applications.

In contrast, the proposed method does not rely on hyperparameter tuning and provides consistent estimation without requiring parameter selection. Therefore, the comparison in this study should be interpreted as highlighting the differences in estimation characteristics between tuning-dependent methods and tuning-free methods.

A summary of the compared methods is provided in Table~\ref{tab:comparison_methods}.

\begin{table}[h]
\centering
\caption{Summary of comparison methods. A checkmark ($\checkmark$) indicates that the corresponding component is required.}
\label{tab:comparison_methods}
\renewcommand{\arraystretch}{1.2}
\begin{tabular}{lccc}
\hline
Method & Regularization & Tuning & Positive definiteness assumption \\ \hline

SEM & -- & -- & $\checkmark$ \\

L1-SEM & $\checkmark$ & $\checkmark$ & $\checkmark$ \\

L2-SEM & $\checkmark$ & $\checkmark$ & $\checkmark$ \\

Proposed method & -- & -- & -- \\

\hline
\end{tabular}
\end{table}

Furthermore, in the $n < p$ regime, the sample covariance matrix becomes singular, making conventional SEM difficult to apply directly. In practical implementations, following the semopy framework, eigenvalue correction is applied to enforce positive definiteness of the covariance matrix before estimation. While this enables computation, it may distort the original covariance structure. For details, see \url{https://semopy.com/ordinal.html}.

\subsubsection{Evaluation Criteria}

To evaluate both the feasibility and accuracy of estimation in high-dimensional small-sample settings, we first assess whether estimation is valid, and then evaluate estimation accuracy only on valid trials.

We define the valid rate as
\begin{align}
\label{eq:valid}
\mathrm{Valid\ Rate}
=
\frac{1}{M}
\sum_{m=1}^M
\mathbf{1}\left(\text{estimation in trial } m \text{ is valid}\right),
\end{align}
where $\mathbf{1}(\cdot)$ denotes the indicator function. In this study, only trials in which estimation is numerically and statistically feasible are considered valid, and subsequent evaluation of estimation accuracy is restricted to these trials.

Since conventional SEM requires the sample covariance matrix to be positive definite, trials in which the original sample covariance matrix is not positive definite are treated as invalid. Although semopy may proceed with estimation using nearPD corrections, such corrections alter the original covariance structure and do not imply that the conventional method is directly applicable to the observed data. Therefore, such cases are excluded from valid trials in our evaluation.

Next, to evaluate the estimation accuracy of the structural coefficient $\beta$, we use the estimators $\{\hat{\beta}^{(m)}\}_{m \in \mathcal{V}}$ obtained from Monte Carlo simulations, where $\mathcal{V}$ denotes the set of valid trials. The mean is defined as
\begin{align}
\bar{\beta}
=
\frac{1}{|\mathcal{V}|}
\sum_{m \in \mathcal{V}}
\hat{\beta}^{(m)}.
\end{align}
We then consider the following metrics:
\begin{align}
\mathrm{Bias} &= \bar{\beta} - \beta,\\
\mathrm{Var} &=
\frac{1}{|\mathcal{V}|}
\sum_{m \in \mathcal{V}}
\left(
\hat{\beta}^{(m)} - \bar{\beta}
\right)^2,\\
\mathrm{RMSE} &=
\sqrt{
\frac{1}{|\mathcal{V}|}
\sum_{m \in \mathcal{V}}
\left(
\hat{\beta}^{(m)} - \beta
\right)^2
}.
\end{align}

RMSE serves as a comprehensive measure of estimation accuracy, reflecting both bias and variance. Bias evaluates systematic error, while variance measures the stability of the estimator. In addition, confidence intervals for these metrics are constructed via bootstrap to visualize uncertainty arising from finite sample sizes.

In this simulation study, we conduct 100 Monte Carlo repetitions. For each repetition, the initialization is performed 10 times, and bootstrap procedures for both the self-covariance and cross-covariance components are conducted 10 times each.

\subsection{Case 2}

In Section~\ref{subsection:pattern1}, we assumed that a finite number of components dominate the factor loading vectors, while the remaining components decay as noise. In this section, we investigate deviations from this assumption. Specifically, as Case 2, we evaluate the behavior of the estimator under gradual violations of sparsity.

We consider the following model:
\begin{equation}
\lambda_{x,r} =
\begin{cases}
1 & r=1,2, \\
a\, p_1^{-\alpha}\cdot \eta_r & r \ge 3,
\end{cases}
\qquad
\lambda_{y,s} =
\begin{cases}
1 & s=1,2, \\
b\, p_2^{-\alpha}\cdot \zeta_s & s \ge 3,
\end{cases}
\qquad
a,b>0,\ \frac{1}{2}<\alpha<1.
\end{equation}

This model extends Equation~\eqref{Pattern1:model} by introducing additional random variables $\eta_r$ and $\zeta_s$, allowing flexible control over the tail behavior of the factor loadings. As a result, the contribution of the non-dominant components can be varied continuously, enabling a systematic evaluation of the breakdown of sparsity.

In this study, we introduce random fluctuations in the tail components by specifying $\eta_r$ and $\zeta_s$ as follows:
\begin{align}
\lambda_{x,r} &=
\begin{cases}
1 & r = 1,2, \\
p_1^{-3/4}\cdot \eta_r & r \ge 3,
\end{cases}
\qquad
\lambda_{y,s} =
\begin{cases}
1 & s = 1,2, \\
p_2^{-3/4}\cdot \zeta_s & s \ge 3,
\end{cases}\\
\Theta_{\delta} &= 0.5\, I_{p_1}, \qquad
\Theta_{\varepsilon} = 0.5\, I_{p_2}, \qquad
\psi = 0.4,\\
\eta_r &\sim \mathrm{Uniform}(1, p_1^{7/4}),\quad
\zeta_s \sim \mathrm{Uniform}(1, p_2^{7/4}).
\end{align}

When $\eta_r=\zeta_s=1$, this model reduces to Equation~\eqref{Pattern1:model}, representing the ideal sparse setting. In contrast, when $\eta_r=p_1^{7/4}$ and $\zeta_s=p_2^{7/4}$, the tail components become of order one, violating the sparsity assumption. This formulation introduces variability in the magnitude of the tail components and enables a gradual transition from sparse to non-sparse structures.

Based on this setup, we conduct a sensitivity analysis to evaluate the relationship between the degree of sparsity and estimation error. The sample size is fixed at $n=7$, and we consider $p_1=p_2=2,3,4,5,6$.

\subsubsection{Comparison Metrics and Evaluation}

As in Case 1, we compare four methods: SEM, L1-SEM, L2-SEM, and the proposed method. The objective of this sensitivity analysis is to assess the robustness of the proposed estimator with respect to violations of the sparsity assumption, particularly in terms of the decay behavior of the tail components.

In the scatter plots, the horizontal axis represents the energy concentration ratio
\[
\frac{\sum_{(r,s)\in C_1}\sigma^2}{\Delta_{xy}},
\]
which quantifies the proportion of total energy explained by the dominant components. Since the maximum value of the factor loadings is normalized to one in this study, strong sparsity leads to concentration of energy in a small number of components, resulting in values close to one. In contrast, in non-sparse settings, energy is more widely distributed, and the ratio approaches zero.

More specifically, when $\eta_r,\zeta_s = O(1)$, the energy is concentrated in the dominant components, and the energy concentration ratio approaches one. On the other hand, when $\eta_r=p_1^{7/4}$ and $\zeta_s=p_2^{7/4}$, the tail components become large, leading to dispersion of energy and a decrease in the concentration ratio toward zero.

The vertical axis represents the estimation error $\widehat{\beta}^{(m)} - \beta$ obtained from Monte Carlo simulations. This allows us to visualize how estimation accuracy changes as the degree of energy concentration decreases.

In particular, this analysis enables a comparison of whether estimation accuracy deteriorates abruptly or gradually as sparsity weakens and the contribution of tail components increases.

In this simulation, we perform 100 Monte Carlo repetitions. For each repetition, initialization is performed 10 times, and bootstrap procedures for both the self-covariance and cross-covariance components are conducted 10 times each.

\section{Results}

\subsection{Case 1}

In this section, we first summarize the valid rate to demonstrate the applicability of each method. Next, to clarify the impact of structural conditions on estimation accuracy, we compare the distributions of $\hat{\beta}^{(m)}$ obtained from Monte Carlo simulations, along with Bias, Variance, and RMSE, under Setting 1 and Setting 2. Finally, to capture overall trends that cannot be fully characterized within each setting, we aggregate results across all dimensional configurations and provide an integrated analysis of the behavior of the proposed method.

\subsubsection{Valid Rate}

Table~\ref{tab:valid_rate} presents the valid rate under each dimensional setting. The valid rate is computed according to Equation~\eqref{eq:valid} and represents the proportion of trials in which estimation was numerically successful and yielded finite estimates for each method.

\begin{table}[h]
\centering
\caption{Valid rate of each method across different dimensional settings ($p_1 = p_2$, $n=10$). In Setting 1 ($n > p_1 + p_2$), all methods achieve valid estimation. In contrast, in Setting 2 ($n > p_1, n > p_2$, but $n \leq p_1 + p_2$), conventional SEM and regularized SEM (L1-SEM, L2-SEM) completely fail due to the singularity of the sample covariance matrix, whereas the proposed method maintains a valid rate of 1.00 across all dimensions.}
\label{tab:valid_rate}
\renewcommand{\arraystretch}{1.2}
\begin{tabular}{c c cccc}
\hline
$p_1 = p_2$ & Setting & SEM & L1-SEM & L2-SEM & Proposed \\ \hline

2 & \multirow{3}{*}{Setting 1} & 1.00 & 1.00 & 1.00 & 1.00 \\
3 &                            & 1.00 & 1.00 & 1.00 & 1.00 \\
4 &                            & 1.00 & 1.00 & 1.00 & 1.00 \\

\hdashline

5 & \multirow{5}{*}{Setting 2} & 0.00 & 0.00 & 0.00 & 1.00 \\
6 &                            & 0.00 & 0.00 & 0.00 & 1.00 \\
7 &                            & 0.00 & 0.00 & 0.00 & 1.00 \\
8 &                            & 0.00 & 0.00 & 0.00 & 1.00 \\
9 &                            & 0.00 & 0.00 & 0.00 & 1.00 \\

\hline
\end{tabular}
\end{table}

From Table~\ref{tab:valid_rate}, it is observed that the proposed method maintains a valid rate of 100\% across all dimensional settings. In contrast, although SEM, L1-SEM, and L2-SEM achieve valid estimation in Setting 1, their valid rates drop to zero in Setting 2 (i.e., $p_1 = p_2 \geq 5$), indicating that estimation fails entirely in this regime.

These results demonstrate that the proposed method preserves feasibility even in $n < p$ settings, whereas conventional covariance-based SEM methods become inapplicable under the same conditions. This highlights a fundamental difference in applicability between the proposed method and existing approaches.

In the following, we present visualizations of $\hat{\beta}^{(m)}$ obtained from Monte Carlo simulations, along with comparisons of Bias, Variance, and RMSE for each setting.

\subsubsection{Setting 1}

We first visualize the distribution of $\hat{\beta}^{(m)}$ obtained from Monte Carlo simulations in Figure~\ref{fig:performance_case1_signed}. The left panel shows results including outliers, while the right panel excludes outliers.

\begin{figure}[h]
\centering
\includegraphics[width=\linewidth]{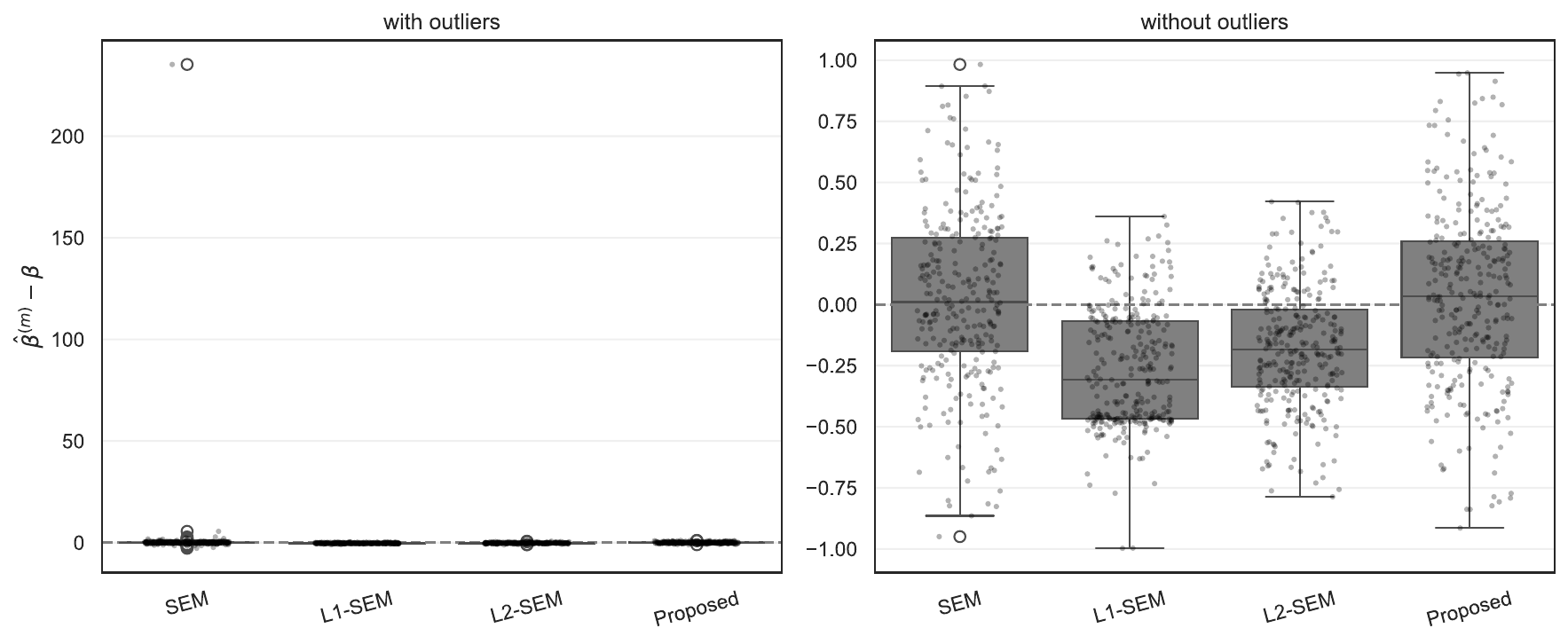}
\caption{Boxplots of estimation errors $\hat{\beta}^{(m)} - \beta$ obtained from Monte Carlo simulations. The left panel includes outliers, while the right panel excludes them. Methods are shown from left to right as SEM, L1-SEM, L2-SEM, and the proposed method.}
\label{fig:performance_case1_signed}
\end{figure}

From Figure~\ref{fig:performance_case1_signed}, several observations can be made regarding the distribution of estimation errors in Setting 1. First, in the left panel including outliers, SEM exhibits extremely large outliers, resulting in a heavy-tailed distribution compared to other methods. In contrast, L1-SEM, L2-SEM, and the proposed method show relatively well-controlled distributions, with most estimation errors concentrated around zero.

Next, focusing on the right panel without outliers, the central tendency of each method becomes clearer. Both SEM and the proposed method exhibit medians close to zero, indicating that the estimated values are, on average, close to the true parameter. On the other hand, L1-SEM and L2-SEM exhibit systematic negative bias, with their distributions shifted toward negative values, implying consistent underestimation.

Overall, in Setting 1, both SEM and the proposed method achieve estimates centered around the true value. However, SEM is highly sensitive to outliers, whereas the proposed method maintains stability. In contrast, L1-SEM and L2-SEM effectively suppress variance but introduce systematic bias.

\begin{table}[h]
\centering
\caption{Performance metrics of each method across different dimensionalities. For each method and dimension $p$, the table reports Bias, Variance (Var), and Root Mean Squared Error (RMSE), along with their corresponding bootstrap confidence intervals (CI). Bias represents systematic deviation from the true parameter, Variance measures the dispersion of the estimator, and RMSE summarizes overall estimation error by combining both Bias and Variance. These metrics should be interpreted jointly to assess the trade-off between accuracy (low Bias), stability (low Variance), and overall performance (low RMSE) across methods.}
\label{tab:performance_metrics}
\renewcommand{\arraystretch}{1.2}
\begin{tabular}{c c c c c c c c}
\hline
Method & $p$
& Bias & Bias CI
& Var & Var CI
& RMSE & RMSE CI \\
\hline
SEM & 2 & 0.133 & (0.064, 0.209) & 0.205 & (0.108, 0.300) & 0.470 & (0.338, 0.576) \\
L1-SEM & 2 & -0.108 & (-0.149, -0.068) & 0.041 & (0.028, 0.064) & 0.228 & (0.184, 0.291) \\
L2-SEM & 2 & 0.010 & (-0.034, 0.071) & 0.076 & (0.052, 0.105) & 0.275 & (0.228, 0.326) \\
Proposed & 2 & 0.108 & (0.047, 0.179) & 0.142 & (0.106, 0.173) & 0.391 & (0.333, 0.431) \\
\hline
SEM & 3 & -0.022 & (-0.126, 0.064) & 0.340 & (0.192, 0.535) & 0.581 & (0.442, 0.729) \\
L1-SEM & 3 & -0.219 & (-0.267, -0.171) & 0.067 & (0.052, 0.086) & 0.338 & (0.292, 0.383) \\
L2-SEM & 3 & -0.273 & (-0.315, -0.238) & 0.044 & (0.034, 0.054) & 0.344 & (0.305, 0.382) \\
Proposed & 3 & 0.009 & (-0.068, 0.082) & 0.163 & (0.116, 0.200) & 0.402 & (0.341, 0.445) \\
\hline
SEM & 4 & 2.380 & (-0.132, 7.163) & 554.121 & (0.283, 1628.785) & 23.542 & (0.530, 40.763) \\
L1-SEM & 4 & -0.465 & (-0.471, -0.460) & 0.001 & (0.001, 0.001) & 0.467 & (0.461, 0.472) \\
L2-SEM & 4 & -0.263 & (-0.306, -0.228) & 0.040 & (0.032, 0.049) & 0.330 & (0.297, 0.367) \\
Proposed & 4 & -0.008 & (-0.084, 0.066) & 0.135 & (0.101, 0.173) & 0.365 & (0.320, 0.416) \\
\hline
\end{tabular}
\end{table}

Table~\ref{tab:performance_metrics} reports Bias, Variance, and RMSE with confidence intervals for each method. For $p=2,3$, L1-SEM and L2-SEM achieve low variance and small RMSE due to regularization, but exhibit consistent negative bias. In contrast, the proposed method maintains near-zero bias, although with higher variance, resulting in moderate RMSE.

At the boundary case $p=4$, the variance and RMSE of SEM increase dramatically, with extremely wide confidence intervals, indicating severe instability. This behavior arises from the singularity of the sample covariance matrix in high-dimensional small-sample settings, demonstrating structural failure of conventional SEM. While L1-SEM and L2-SEM remain stable, they continue to exhibit strong bias. In contrast, the proposed method maintains near-zero bias and controlled variance, achieving a balance between stability and structural validity.

\subsubsection{Setting 2}

We next visualize $\hat{\beta}^{(m)}$ in Figure~\ref{fig:performance_case2_signed}. Since only the proposed method is applicable in Setting 2 (as shown in Table~\ref{tab:valid_rate}), boxplots are presented only for the proposed method.

\begin{figure}[h]
\centering
\includegraphics[width=\linewidth]{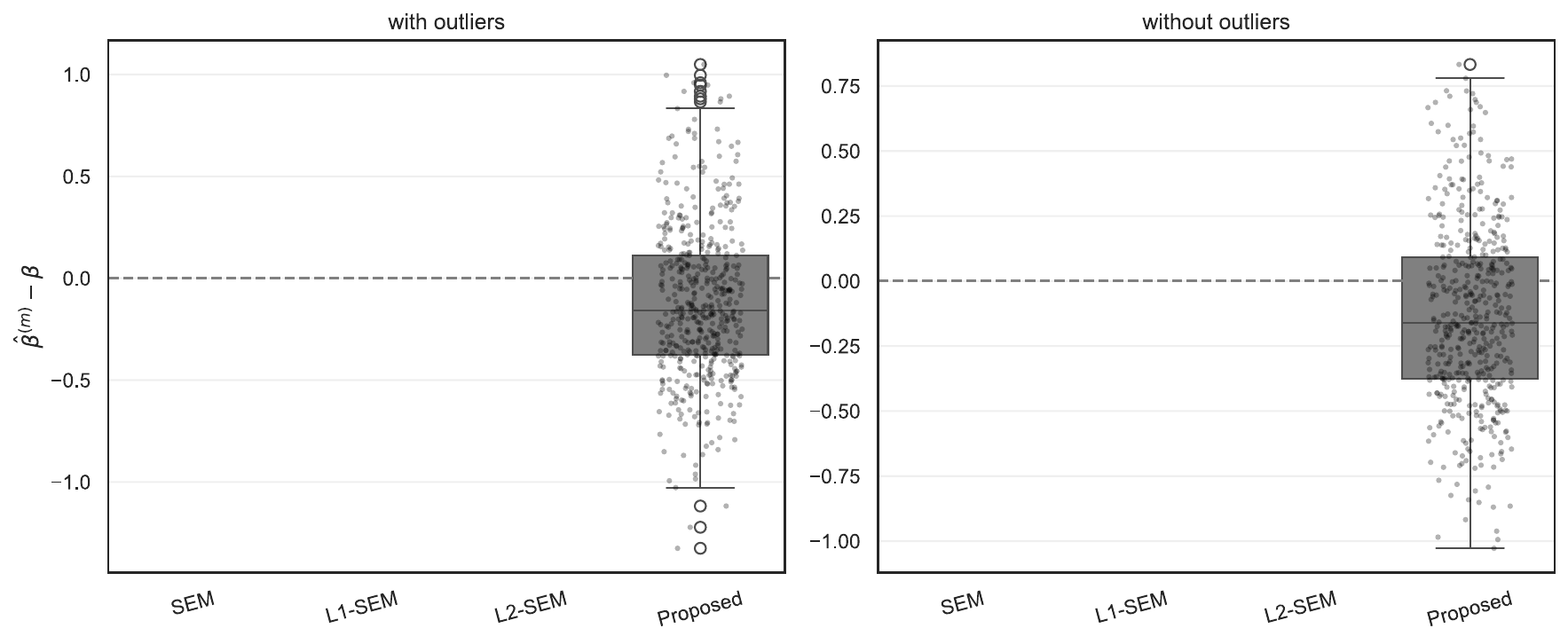}
\caption{Boxplots of estimation errors $\hat{\beta}^{(m)} - \beta$ in Setting 2. The left panel includes outliers, and the right panel excludes them. Note that SEM, L1-SEM, and L2-SEM are not shown because estimation failed under this setting, and only the proposed method yields valid results.}
\label{fig:performance_case2_signed}
\end{figure}

From Figure~\ref{fig:performance_case2_signed}, we observe that the distribution of estimation errors remains relatively compact even when outliers are included, and extreme outliers observed in Setting 1 do not appear. However, the distribution exhibits a noticeable spread, particularly with several negative outliers.

In the right panel excluding outliers, the median is shifted in the negative direction, indicating systematic underestimation. Compared to Case 1, this negative bias is more pronounced, suggesting that the estimator becomes biased as structural conditions deviate from the ideal setting.

Next, Figure~\ref{fig:performance_case2} summarizes the performance of the proposed method across dimensions.

\begin{figure}[h]
\centering
\includegraphics[width=0.9\linewidth]{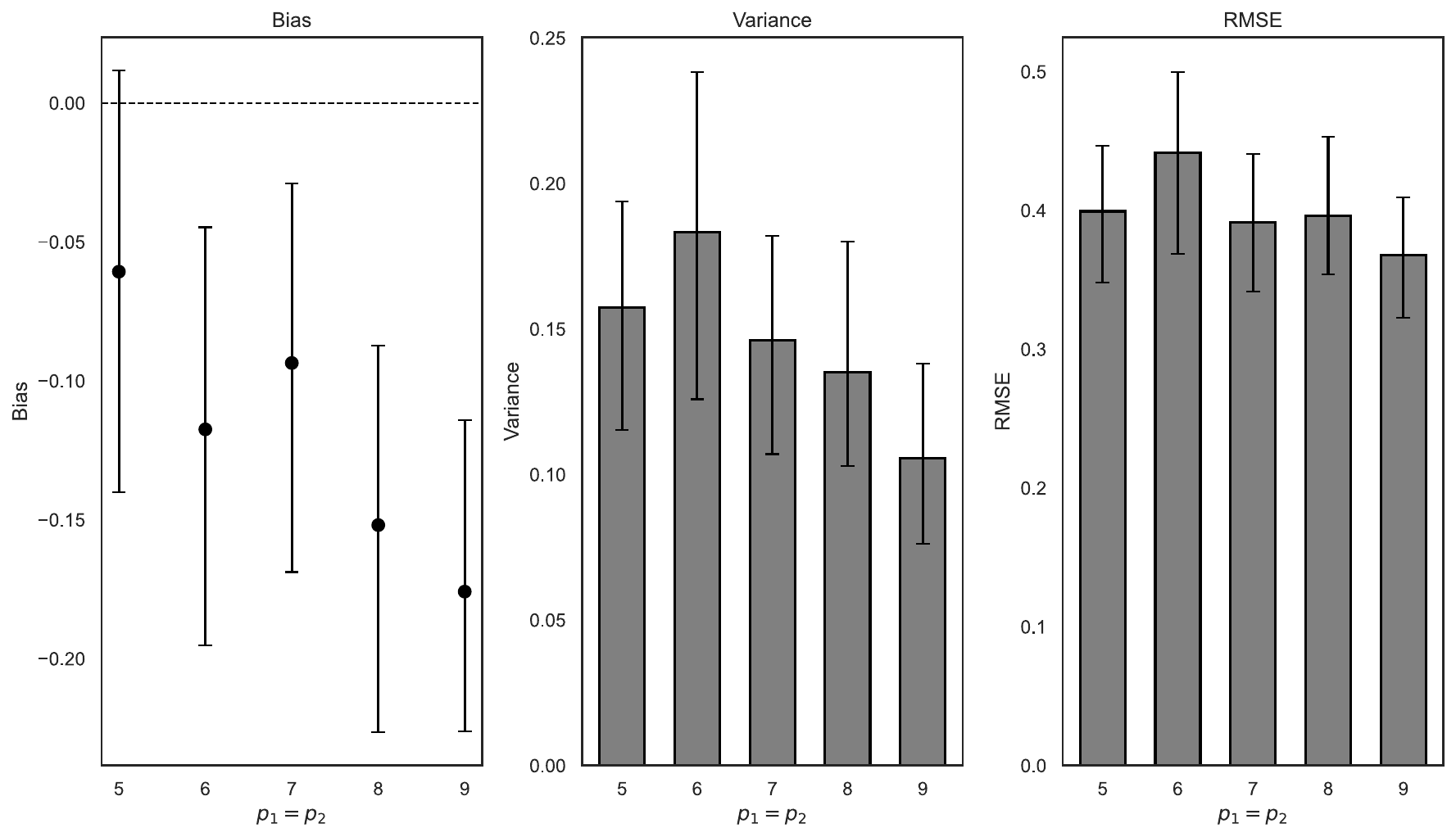}
\caption{Performance of the proposed method under Setting 2. Bias, variance, and RMSE are shown with their 95\% confidence intervals across different dimensionalities ($p_1 = p_2$).}
\label{fig:performance_case2}
\end{figure}

Bias becomes increasingly negative as dimensionality increases, indicating stronger underestimation in higher dimensions. The confidence intervals do not include zero, confirming that the bias is consistently negative.

Variance reaches its maximum around $p=6$ but decreases as dimensionality further increases, suggesting that variability of the estimator is reduced in higher dimensions. RMSE also peaks around $p=6$ but remains relatively stable overall, showing only moderate variation across dimensions.

These results indicate that, although the proposed method introduces bias under structural violations, it avoids catastrophic failure and maintains controlled estimation error, demonstrating robustness against deviations from sparsity assumptions.

\subsubsection{Behavior of the Proposed Method}

As shown in Table~\ref{tab:valid_rate}, the proposed method is the only approach that remains computationally feasible under both Setting 1 and Setting 2, whereas all competing methods fail under high-dimensional conditions. Therefore, we focus on the proposed method and analyze its behavior by visualizing Bias, Variance, and RMSE across both settings in a unified manner. The results are presented in Figure~\ref{fig:performance_all}.

\begin{figure}[h]
\centering
\includegraphics[width=0.9\linewidth]{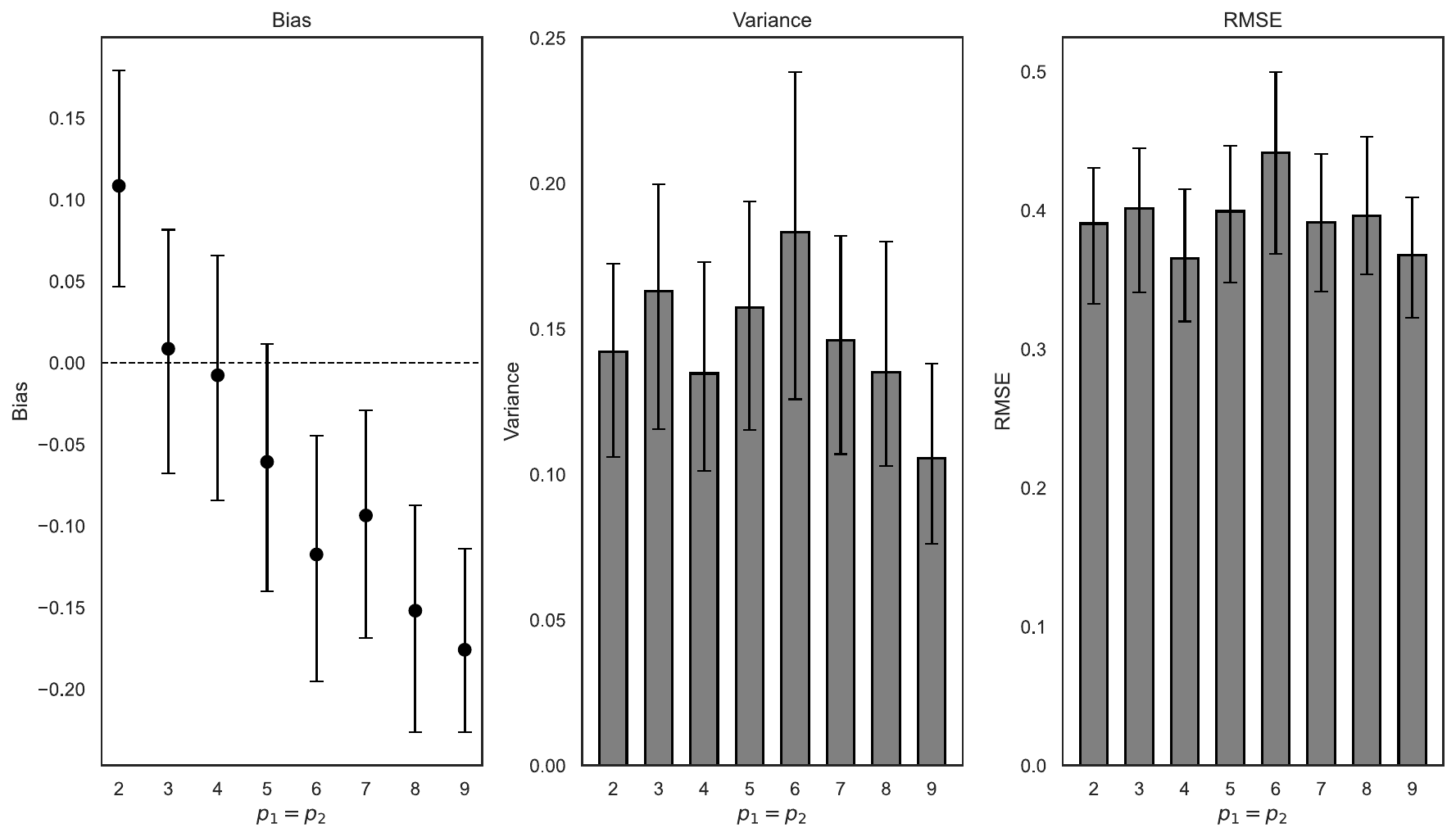}
\caption{Performance of the proposed method under Setting 1 and Setting 2. Bias, variance, and RMSE are shown with their 95\% confidence intervals across different dimensionalities ($p_1 = p_2$).}
\label{fig:performance_all}
\end{figure}

From Figure~\ref{fig:performance_all}, several key patterns can be observed. First, Bias decreases monotonically as $p$ increases from $2$ to $9$. Although a decreasing trend was already observed within Setting 2, the unified visualization across both settings reveals that this trend holds consistently over the entire range of dimensionalities.

Second, Variance does not exhibit a monotonic trend. It increases from $p=2$ to approximately $p=6$, reaches its maximum around $p=6$, and then decreases for $p \geq 7$. Despite this non-monotonic behavior, the width of the confidence intervals remains relatively stable across all dimensions, indicating that the variability of the estimator is well controlled.

Finally, RMSE remains within a relatively narrow range across all dimensionalities, without showing a clear monotonic pattern. Although it slightly increases around $p=6$, it subsequently decreases as $p$ increases further, suggesting that the overall estimation error remains stable even in higher-dimensional regimes.

These results indicate that the proposed method maintains controlled estimation behavior across both feasible and infeasible regimes for conventional methods. In particular, although bias gradually increases in magnitude, the overall estimation error does not deteriorate drastically, demonstrating stable performance across varying dimensional conditions.

\subsection{Case 2}

\subsubsection{Visualization of Results}

Figure~\ref{fig:case2_result} presents scatter plots of the estimation error $\widehat{\beta}^{(m)} - \beta$ obtained from Monte Carlo simulations. Rows correspond to SEM, L1-SEM, L2-SEM, and the proposed method (from top to bottom), while columns correspond to $p=2,3,4,5,6$ (from left to right). The vertical axis is restricted to the range $[-1.5, 1.5]$ for visualization purposes.

\begin{figure}[h]
\centering
\includegraphics[width=\linewidth]{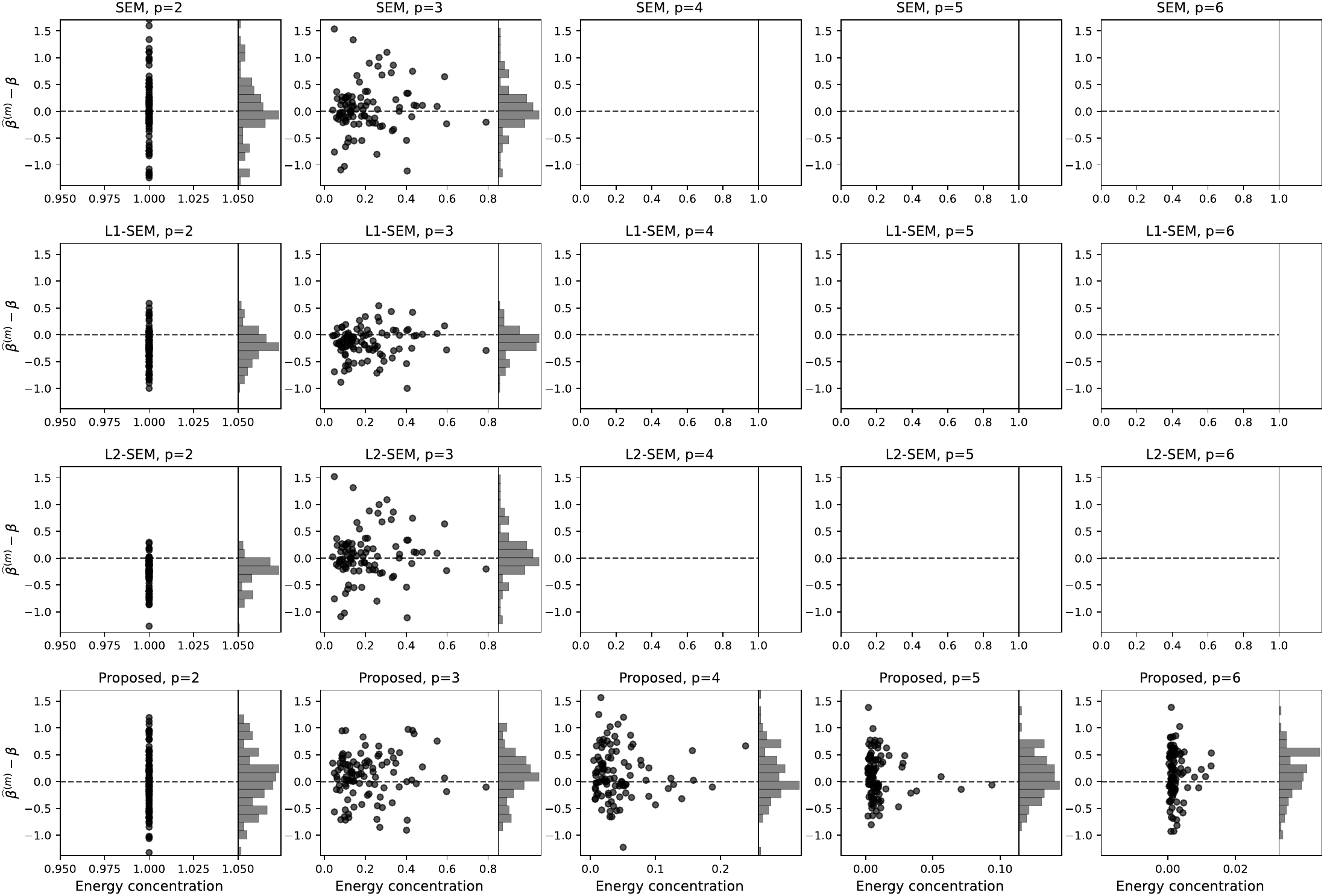}
\caption{Scatter plots of estimation error $\widehat{\beta}^{(m)}-\beta$ against energy concentration ratio under Case 2. Rows: SEM, L1-SEM, L2-SEM, and proposed method (top to bottom). Columns: $p_1=p_2=2,3,4,5,6$ (left to right). The vertical axis range is restricted to $[-1.5, 1.5]$.}
\label{fig:case2_result}
\end{figure}

In addition, Table~\ref{tab:sign_ratio_wide} summarizes whether the estimation error $\widehat{\beta}^{(m)} - \beta$ is positive or negative. The positive rate (pos) denotes the proportion of trials in which the error is greater than zero, while the negative rate (neg) denotes the proportion of trials in which the error is less than zero.

\begin{table}[h]
\centering
\caption{Sign ratio of estimation error by dimension and method. For each method and dimension $p$, ``pos'' and ``neg'' denote the proportions of Monte Carlo trials in which the estimation error $\hat{\beta}^{(m)} - \beta$ is positive and negative, respectively. Values are computed over valid trials only. Entries marked with ``--'' indicate that estimation failed and no valid results were obtained. This table provides insight into the directional bias of estimators, where deviations from a balanced ratio (0.5, 0.5) indicate systematic overestimation or underestimation.}
\label{tab:sign_ratio_wide}
\renewcommand{\arraystretch}{1.2}
\begin{tabular}{c cc cc cc cc}
\hline
 & \multicolumn{2}{c}{SEM}
 & \multicolumn{2}{c}{L1-SEM}
 & \multicolumn{2}{c}{L2-SEM}
 & \multicolumn{2}{c}{Proposed} \\
$p$
 & pos & neg
 & pos & neg
 & pos & neg
 & pos & neg \\
\hline
2 & 0.52 & 0.48 & 0.23 & 0.77 & 0.11 & 0.89 & 0.51 & 0.49 \\
3 & 0.53 & 0.47 & 0.28 & 0.72 & 0.53 & 0.47 & 0.65 & 0.35 \\
4 & -    & -    & -    & -    & -    & -    & 0.57 & 0.43 \\
5 & -    & -    & -    & -    & -    & -    & 0.57 & 0.43 \\
6 & -    & -    & -    & -    & -    & -    & 0.67 & 0.33 \\
\hline
\end{tabular}
\end{table}

From Figure~\ref{fig:case2_result}, for $p=2$, the energy concentration ratio takes a constant value, resulting in the disappearance of variation along the horizontal axis. Consequently, the scatter plot degenerates into a vertical structure. This behavior arises because only dominant components exist in this setting, making the energy concentration ratio constant. Therefore, for $p=2$, distributional summaries are more informative than scatter plots.

As shown in Table~\ref{tab:sign_ratio_wide}, L1-SEM and L2-SEM exhibit predominantly negative errors, whereas SEM and the proposed method show a more balanced distribution of positive and negative errors.

For $p \geq 3$, the proposed method successfully produces estimates across all dimensions. In contrast, for SEM, L1-SEM, and L2-SEM, the number of cases in which estimation fails increases as dimensionality grows, resulting in regions where scatter plots cannot be obtained.

Furthermore, for the proposed method, the proportion of positive errors tends to increase with dimensionality. Specifically, the positive ratio increases from 0.51 to 0.67 as $p$ increases, indicating a gradual shift in the error distribution. In contrast, L1-SEM and L2-SEM remain dominated by negative errors across dimensions.

Next, we examine the shape of the error distributions. To quantify the concentration of estimation errors, we compute the interquartile range (IQR) for each method. The IQR represents the range of the central 50\% of the distribution and serves as a robust measure of dispersion.

\begin{table}[h]
\centering
\caption{IQR of estimation error by dimension and method. For each method and dimension $p$, the IQR summarizes the dispersion of the estimation error $\hat{\beta}^{(m)} - \beta$ as the difference between the 75th and 25th percentiles, providing a robust measure of variability that is less sensitive to outliers than variance. Smaller IQR values indicate more stable estimation. Entries marked with ``--'' indicate that estimation failed and no valid results were obtained.}
\label{tab:iqr_wide}
\renewcommand{\arraystretch}{1.2}
\begin{tabular}{c c c c c}
\hline
$p$
 & SEM
 & L1-SEM
 & L2-SEM
 & Proposed \\
\hline
2 & 0.6347 & 0.3653 & 0.2587 & 0.6684 \\
3 & 0.4550 & 0.2901 & 0.4546 & 0.4529 \\
4 & -      & -      & -      & 0.5753 \\
5 & -      & -      & -      & 0.5766 \\
6 & -      & -      & -      & 0.6244 \\
\hline
\end{tabular}
\end{table}

From Table~\ref{tab:iqr_wide}, for $p=2,3$, L1-SEM and L2-SEM exhibit relatively small IQR values, indicating strong concentration of estimates around the center. In contrast, the proposed method shows slightly larger IQR values, suggesting a wider spread of estimation errors.

Moreover, for the proposed method, the IQR increases gradually with dimensionality, indicating that the concentration of the error distribution decreases as dimensionality increases. This reflects a gradual degradation in estimation precision rather than abrupt failure.

\section{Application}

From the simulation results, it is confirmed that the proposed method produces stable estimates even when the sparsity assumption is violated, that is, in regions with low energy concentration. Furthermore, in several cases, the lower bound of the bootstrap-based 95\% confidence interval exceeds zero, suggesting stability in the sign of the estimated parameter. From a practical perspective, this indicates that the method is useful for determining the direction of influence.

On the other hand, in such cases, the estimation may rely primarily on the self-covariance component, and thus structural validity and causal interpretation are not guaranteed. Therefore, caution is required when interpreting the estimated relationships.

As a real-data application, we apply the proposed method to the Colombia subsample of the DASS-42 (Depression Anxiety Stress Scales) dataset to estimate the latent path coefficient $\beta$ from Stress to Depression. The DASS-42 is a questionnaire consisting of 42 items designed to measure three psychological states: Depression, Anxiety, and Stress. The dataset is publicly available at \url{https://openpsychometrics.org/_rawdata/}, with each scale comprising 14 items. The data were collected between 2017 and 2019, with a total of 39,775 responses across 146 countries.

In this study, we focus on estimating the path coefficient from Stress to Depression. The analysis is conducted using the Colombia subsample, which contains $n=27$ observations. In this setting, the dimensionality of each block is $p_1=14$ and $p_2=14$, resulting in $n < p_1 + p_2$.

For estimation, we perform 10 random initializations, and for each initialization, we conduct 10 bootstrap resamples for both the self-covariance and cross-covariance components. Additionally, for statistical inference of the coefficient, we apply the bootstrap-based method described in Section~\ref{subsection:test}, using 100 bootstrap repetitions.

The estimation results are summarized in Table~\ref{tab:bootstrap_results}. We focus on the estimated coefficient $\beta$. The estimate is $\widehat{\beta}=0.513$, indicating a positive directional relationship from Stress to Depression. Furthermore, the bootstrap results confirm statistical significance.

\begin{table}[h]
\centering
\caption{Bootstrap estimation results for the structural parameter $\beta$. The table reports the point estimate $\hat{\beta}$, along with the bootstrap mean and standard deviation, which summarize the sampling distribution obtained via resampling. The 95\% confidence interval is constructed using the bootstrap distribution, and statistical significance is assessed based on whether zero is included in the interval. The minimum SRMR value indicates the goodness-of-fit of the selected model, where smaller values correspond to better fit.}
\label{tab:bootstrap_results}
\renewcommand{\arraystretch}{1.2}
\begin{tabular}{lc}
\hline
Metric & Value \\
\hline
$\hat{\beta}$ & 0.513 \\
Bootstrap mean & 0.601 \\
Bootstrap standard deviation & 0.112 \\
95\% confidence interval & [0.392, 0.825] \\
Significance & $p=0.000$, significant \\
Minimum SRMR & 0.149 \\
\hline
\end{tabular}
\end{table}

These results indicate a positive dependency from Stress to Depression. In the existing literature, stress has been identified as a major trigger for both the onset and recurrence of depression, and this relationship has been consistently supported by empirical studies \cite{hammen2005stress}. Furthermore, even after controlling for genetic factors and shared environmental influences, stressful life events have been shown to significantly predict the onset of major depression \cite{kendler1999causal}. The positive path from Stress to Depression estimated in this study is consistent with these established findings.

\section{Discussion}

In this section, we discuss the proposed method from the perspectives of feasibility of estimation, behavior under both model validity and violation, structural characteristics of estimation error, and practical usefulness and limitations. In particular, through comparisons with existing methods in small-sample settings with $p > n$, we clarify under which conditions the proposed method functions effectively and what caveats are required in its interpretation.

The results on the valid rate demonstrate that the proposed method consistently produces solutions even in small-sample settings with $p > n$. Notably, while existing methods fail to produce estimates under such conditions, the proposed method maintains a high valid rate. This indicates a clear advantage in terms of feasibility. This behavior is consistent with the fact that the proposed framework does not rely on likelihood-based estimation and thus avoids the breakdown caused by the singularity of the sample covariance matrix.

From the results of Case 1, where the model assumptions hold, the proposed method provides estimates that are comparable to or more stable than those obtained by existing methods. However, an examination of the error decomposition reveals that the estimation error tends to exhibit a bias toward underestimation, while the variance remains relatively small. This suggests that the error structure of the proposed method is primarily dominated by bias rather than variance. In other words, the method suppresses variability at the cost of introducing systematic deviation.

Taken together, these results imply that although the proposed method is clearly superior in terms of producing feasible solutions, careful attention is required regarding the nature of the obtained estimates. In particular, the bias-dominant behavior suggests that structural distortion may be introduced in exchange for stability.

From the results of Case 2, where the sparsity assumption is violated, it is confirmed that the proposed method continues to produce estimates. Moreover, these estimates do not diverge but remain within a controlled range. This phenomenon can be interpreted as the method estimating a pseudo-structure that satisfies the imposed constraints, even when the true structure does not. Therefore, the obtained estimates should not be interpreted as direct estimates of the true structure, but rather as the most consistent approximation under the given constraints.

Furthermore, simulation results indicate that even in regions where sparsity does not hold, that is, where the energy concentration ratio is low, the proposed method continues to produce stable estimates. In addition, cases were observed in which the lower bound of the bootstrap-based 95\% confidence interval exceeds zero, suggesting stability in the sign of the estimated coefficient. This implies that, although the magnitude of the coefficient may not be reliable, the method may still be useful for determining the direction of influence in practical settings. In this sense, the method can provide relatively stable information regarding the sign and direction of structural parameters, even in environments where conventional assumptions do not hold.

However, in such situations, the estimated values should be interpreted as approximations within the feasible constraint set rather than direct reflections of the true underlying structure. Therefore, even when the sign is stable, it should not be directly interpreted as evidence of causal relationships or structural validity, but rather as an exploratory or auxiliary indicator.

In addition, the increase in the negative ratio observed in Case 2 suggests that estimation becomes less stable as sparsity weakens. This result indicates that the sparsity assumption plays a fundamental role in ensuring estimation stability, and that deviations from this assumption directly affect estimation behavior.

From the application results, it is suggested that the proposed method enables the evaluation of whether relationships established in large-sample settings persist in small-sample environments. In other words, the method provides a framework for reassessing the robustness of findings that were previously dependent on large samples.

However, it should be noted that the empirical analysis is based on a relatively small sample ($n=27$) from Colombia, which imposes limitations on external validity. In particular, DASS-42 responses are known to be influenced by cultural factors and response tendencies, and therefore the results may not generalize directly to other populations. In addition, since the data are based on an online survey, the possibility of selection bias cannot be ruled out. Furthermore, the results depend on the specification of the latent variable model. Specifically, the magnitude and sign of the estimated path coefficient may vary depending on the factor loading structure and identification constraints. Although certain loadings are fixed for identification purposes in this study, the impact of these assumptions on the estimation results cannot be entirely excluded.

Finally, several directions for future research remain. Since the estimation of the factor loading vectors $\Lambda_x$ and $\Lambda_y$ relies on the self-covariance component, their accuracy is strongly dependent on the estimation quality of this component. Therefore, improving the estimation of self-covariance is a critical issue for enhancing overall performance. Moreover, the non-convex nature of the estimation problem increases sensitivity to initialization and search strategies, potentially leading to variability in the obtained solutions. Designing robust optimization procedures is thus an important topic for further investigation. In addition, extending the framework to multi-factor models and relaxing the current assumptions, such as $n < p_1 + p_2$ and $n > p_1, p_2$, toward more general settings (e.g., $n < \sum p_i$ and $n > p_i$) remains an important direction for future work.

\section{Conclusion}

In this study, we proposed a new estimation framework for covariance-based SEM to address the breakdown of statistical inference in small-sample settings with $p > n$. Conventional likelihood-based approaches become undefined due to the singularity of the sample covariance matrix. In contrast, the proposed method circumvents this difficulty by decomposing the estimation problem into self-covariance and cross-covariance components and treating them based on their distinct statistical properties. In particular, by combining the construction of a feasible set based on self-covariance with a relative error constraint on cross-covariance, the proposed method enables stable estimation even in high-dimensional settings.

Furthermore, by introducing a one-factor structure, issues of identifiability are resolved with minimal constraints, thereby reducing estimation instability. For the nonlinear and non-smooth estimation procedure, we also proposed a bootstrap-based framework for constructing confidence intervals, allowing uncertainty quantification without relying on conventional asymptotic theory.

The results of numerical experiments and real-data analysis demonstrate that the proposed method maintains a certain degree of stability even when the sparsity assumption is not strictly satisfied. In particular, it is shown to be practically useful for limited purposes such as determining the sign of coefficients. On the other hand, as the assumptions on the cross-covariance structure weaken, both bias and variability of the estimator tend to increase, indicating that the performance of the method depends on underlying structural conditions.

The main contribution of this study lies in reconstructing the feasibility of estimation within the covariance-based SEM framework under small-sample settings with $p > n$. However, the proposed method relies on structural constraints such as the one-factor model and rank-one cross-covariance structure. Extending the framework to multi-factor models and more general settings remains an important direction for future research. In addition, further investigation is required to establish theoretical properties of the estimator, including finite-sample error bounds and rigorous asymptotic analysis.

In conclusion, this study presents a new estimation paradigm for SEM in small-sample settings with $p > n$, providing a foundation for future theoretical developments and practical applications.

\section{Acknowledgment}
This work was supported by Japan Science and Technology Agency, Support for Pioneering Research Initiated by the Next Generation (no.JPMJSP2124) and Japan Society for the Promotion of Science Grants-in-Aid for Scientific Research (no.23K28192).

\appendix

\section{Verification of the Sparsity Condition}
\label{appendixA}

We assume that $p_1$ and $p_2$ are of the same order, i.e., $p_1 \asymp p_2$.
\begin{proposition}
Let the factor loadings be defined as
\begin{equation}
\lambda_{x,r} =
\begin{cases}
1 & r=1,2, \\
a\,p_1^{-\alpha} & r \ge 3,
\end{cases}
\quad
\lambda_{y,s} =
\begin{cases}
1 & s=1,2, \\
b\,p_2^{-\alpha} & s \ge 3,
\end{cases}
\quad
a,b>0,\ \frac{1}{2}<\alpha<1.
\end{equation}
Define
\[
\sigma_{(r,s)}^2 \asymp \lambda_{x,r}^2 \lambda_{y,s}^2,
\qquad
\Delta_{xy} = \sum_{(r,s)\in C} \sigma_{(r,s)}^2.
\]
Then, the sparsity condition~\eqref{eq:sparsity_condition} holds.
\end{proposition}

\begin{proof}
Let $C_1 = \{1,2\} \times \{1,2\}$. For $(r,s) \in C_1$, we have $\sigma_{(r,s)}^2 \asymp 1$, and thus
\[
\liminf_{p \to \infty} \min_{(r,s)\in C_1} \sigma_{(r,s)}^2 > 0.
\]

Next, we decompose $\Delta_{xy}$ as
\[
\Delta_{xy}
=
\sum_{(r,s)\in C_1} \sigma_{(r,s)}^2
+
\sum_{(r,s)\in C \setminus C_1} \sigma_{(r,s)}^2.
\]

We evaluate each term in $C \setminus C_1$. For $r \ge 3$ and $s \in \{1,2\}$, we have
\[
\sum_{r \ge 3,\, s \in \{1,2\}} \sigma_{(r,s)}^2
=
O(p_1^{1-2\alpha})
=
o(1),
\]
since $\alpha > \frac{1}{2}$. Similarly, for $r \in \{1,2\}$ and $s \ge 3$,
\[
\sum_{r \in \{1,2\},\, s \ge 3} \sigma_{(r,s)}^2
=
O(p_2^{1-2\alpha})
=
o(1).
\]

Furthermore, for $r \ge 3$ and $s \ge 3$,
\[
\sum_{r \ge 3,\, s \ge 3} \sigma_{(r,s)}^2
=
O(p_1^{1-2\alpha} \, p_2^{1-2\alpha})
=
o(1).
\]

Therefore,
\[
\sum_{(r,s)\in C \setminus C_1} \sigma_{(r,s)}^2 = o(1),
\quad
\Delta_{xy}
=
\sum_{(r,s)\in C_1} \sigma_{(r,s)}^2 + o(1).
\]

It follows that
\[
\sum_{(r,s)\in C_1}
\frac{\sigma_{(r,s)}^2}{\Delta_{xy}}
=
1 + o(1).
\]

Finally, from the assumed structure of the factor loadings, we obtain
\[
\max_{(r,s)\in C \setminus C_1} \sigma_{(r,s)}^2 = o(1).
\]

Combining these results, all conditions in \eqref{eq:sparsity_condition} are satisfied, which completes the proof.
\end{proof}

\bibliographystyle{tfq}
\bibliography{interacttfqsample.bib}

\end{document}